\documentclass[letterpaper,10pt,conference]{ieeeconf}
\IEEEoverridecommandlockouts

\usepackage[T1]{fontenc}
\usepackage{tgtermes}

\usepackage{amsmath}
\usepackage{amsfonts}
\usepackage{amssymb}
\usepackage{cases}
\usepackage{dsfont}

\usepackage{graphicx}
\usepackage{subcaption}
\usepackage{tikz}

\usepackage{array}
\usepackage{booktabs}
\usepackage{multirow}
\usepackage{makecell}
\usepackage{tablefootnote}
\usepackage{threeparttable}

\usepackage[linesnumbered,ruled]{algorithm2e}

\usepackage{pifont}
\usepackage{bbding}
\usepackage{textcomp}
\usepackage{romannum}

\usepackage{changepage}
\usepackage{stfloats}
\usepackage{verbatim}
\usepackage{url}
\usepackage{xcolor}
\usepackage{cite}

\definecolor{pink}{rgb}{0.58,0,0.83}
\definecolor{orange}{rgb}{1,0.5,0}
\definecolor{lightgreen}{rgb}{0.2,0.8,0.2}
\definecolor{lightyellow}{rgb}{0.84,0.65,0.13}

\newcommand{\cmark}{\ding{51}}
\newcommand{\xmark}{\ding{55}}

\usepackage[
    colorlinks,
    linkcolor=blue,
    anchorcolor=blue,
    citecolor=red
]{hyperref}

\usepackage[capitalise]{cleveref}

\crefname{figure}{Fig.}{Figs.}
\Crefname{figure}{Fig.}{Figs.}

\crefname{subfigure}{Fig.}{Figs.}
\Crefname{subfigure}{Fig.}{Figs.}

\crefrangelabelformat{subfigure}{%
  #3#1#4--#5(\crefstripprefix{#1}{#2}#6%
}

\newcounter{definition}
\renewcommand{\thedefinition}{\arabic{definition}}
\newenvironment{definition}[1][]{%
  \refstepcounter{definition}%
  \par\smallskip\noindent
  \textbf{Definition~\thedefinition%
    \if\relax\detokenize{#1}\relax\else\ (\normalfont #1)\fi.}\normalfont
}{\par}

\newcounter{problem}
\renewcommand{\theproblem}{\arabic{problem}}
\newenvironment{problem}[1][]{%
  \refstepcounter{problem}%
  \par\smallskip\noindent
  \textbf{Problem~\theproblem%
    \if\relax\detokenize{#1}\relax\else\ (\normalfont #1)\fi.~}\normalfont
}{\par}

\newcounter{assumption}
\renewcommand{\theassumption}{\arabic{assumption}}
\newenvironment{assumption}[1][]{%
  \refstepcounter{assumption}%
  \par\smallskip\noindent
  \textbf{Assumption~\theassumption%
    \if\relax\detokenize{#1}\relax\else\ (\normalfont #1)\fi.~}\normalfont
}{\par}

\newcounter{lemma}
\renewcommand{\thelemma}{\arabic{lemma}}
\newenvironment{lemma}[1][]{%
  \refstepcounter{lemma}%
  \par\smallskip\noindent
  \textbf{Lemma~\thelemma%
    \if\relax\detokenize{#1}\relax\else\ (\normalfont #1)\fi.~}\itshape
}{\par}

\newcounter{corollary}
\renewcommand{\thecorollary}{\arabic{corollary}}
\newenvironment{corollary}[1][]{%
  \refstepcounter{corollary}%
  \par\smallskip\noindent
  \textbf{Corollary~\thecorollary%
    \if\relax\detokenize{#1}\relax\else\ (\normalfont #1)\fi.~}\itshape
}{\par}

\newcounter{theorem}
\renewcommand{\thetheorem}{\arabic{theorem}}
\newenvironment{theorem}[1][]{%
  \refstepcounter{theorem}%
  \par\smallskip\noindent
  \textbf{Theorem~\thetheorem%
    \if\relax\detokenize{#1}\relax\else\ (\normalfont #1)\fi.~}\itshape
}{\par}

\newcounter{property}
\renewcommand{\theproperty}{\arabic{property}}

\newcounter{remark}
\renewcommand{\theremark}{\arabic{remark}}
\newenvironment{remark}[1][]{%
  \refstepcounter{remark}%
  \par\smallskip\noindent
  \textit{Remark~\theremark%
    \if\relax\detokenize{#1}\relax\else\ (\normalfont #1)\fi.~}\normalfont
}{\par}

\renewenvironment{proof}{%
  \par\smallskip\noindent
  \textit{Proof:}\enspace
}{%
  \hfill$\square$\par
}

\usepackage{arydshln}
\usepackage{multicol}
\usepackage{soul}
\usepackage{tablefootnote} 
\definecolor{ForestGreen}{RGB}{34,139,34}
\usepackage{balance}
\begin{document}
	\title{\bf Decoupling Physical Speed from Path Parameterization in Singularity-Free Guiding Vector Fields}
\author{
Zhouru Xiao, Sha Luo, Yang Lu, Mingliang Xiao, Weijia Yao, Bohuan Lin, Xianzhe Cheng, and Yaonan Wang
\thanks{Zhouru Xiao, Mingliang Xiao, Weijia Yao, and Yaonan Wang are with the School of Artificial Intelligence and Robotics, Hunan University, Changsha, China. Sha Luo is with the College of Information Science and Engineering, Hunan Normal University, Changsha, China. Yang Lu is with the College of Systems Engineering and the State Key Laboratory of Digital Intelligent Modeling and Simulation, National University of Defense Technology, Changsha, China. Bohuan Lin is with Xi'an Jiaotong-Liverpool University, Suzhou, China. Xianzhe Cheng is with Information Support Force Engineering University, Xi'an, China. (e-mail: xzr798@hnu.edu.cn; luosha@hunnu.edu.cn; luyang18@mail.sdu.edu.cn; Mingliangxiao@hnu.edu.cn; wjyao@hnu.edu.cn; bohuan.lin@xjtlu.edu.cn; chengxianzhe22@isfeu.edu.cn; yaonan@hnu.edu.cn). \textit{(Corresponding author: Sha Luo).}}
}
	
	\maketitle
	\thispagestyle{empty}
	\pagestyle{empty}
	
	\begin{abstract}
The existing singularity-free guiding vector field (SF-GVF) with an additional virtual coordinate can eliminate singular points (i.e., points where the vector field vanishes) inherent in conventional GVFs and guarantee global convergence of robot trajectories to closed and self-intersecting desired paths. However, the desired speed given by the GVF along the desired path in the original lower-dimensional space cannot be arbitrarily specified but depends on path parameterizations. One possible workaround is to partially normalize the physical projection of the SF-GVF and assign a user-designed speed. However, we show that this workaround may introduce new singularities since the normalization denominator can become zero. To address this issue, we propose a new SF-GVF with prescribed physical speed (PPS). The integral curves of the new SF-GVF converge \textit{exponentially} to the desired path from any initial condition in the higher-dimensional space (including virtual dimension); more importantly, the robot's physical speed converges to the PPS, while the path-error dynamics remain \textit{invariant} under regular reparameterizations of the desired path. We further develop a saturated acceleration control law for second-order kinematic models. Finally, comparative simulations and 3D path-following experiments with a quadrotor under different PPS profiles validate the theoretical results and demonstrate the effectiveness of the proposed approach. Additional experimental results and videos are available at: {\color{blue}\url{https://anonymous-798.github.io/SF-GVF_with_PPS/}}.
	\end{abstract}
	
	
	\definecolor{limegreen}{rgb}{0.2, 0.8, 0.2}
	\definecolor{forestgreen}{rgb}{0.13, 0.55, 0.13}
	\definecolor{greenhtml}{rgb}{0.0, 0.5, 0.0}
	
\section{Introduction}
Path following is a fundamental task in mobile robot navigation, with applications in UAVs \cite{sujit2014survey,rubi2020survey}, marine vehicles \cite{fossen2003los}, and other robotic systems \cite{samson1995chained,aguiar2007trajectory}. Unlike trajectory tracking, where a robot is required to track a time-dependent reference point, the desired path in a path-following problem is usually given as a geometric curve \textit{without temporal information} \cite{aguiar2007trajectory,aguiar2008limitations}. The robot is then guided to converge to the desired path from a given initial point and move along it. Different approaches have been developed for path following, including projection-based methods \cite{samson1995chained,aguiar2007trajectory}, line-of-sight guidance \cite{fossen2003los}, and vector field (VF) methods \cite{nelson2007vector,michalek2010vfo,yao2020path3d}. 

Among different path-following methods, those using a \textit{guiding vector field (GVF)} have been widely studied due to their intuitive formulation, computational simplicity, and convenient integration with robot controllers \cite{kapitanyuk2018gvf,goncalves2010vector,nelson2007vector,michalek2010vfo,yao2020path3d}. The basic idea is to design a vector field whose integral curves converge to the desired path, so that the robot can be guided to approach and move along the path asymptotically \cite{kapitanyuk2018gvf,goncalves2010vector,yao2020path3d}. Based on this idea, path-following controllers have been developed for nonholonomic mobile robots (e.g., \cite{goncalves2010vector,michalek2010vfo}) and fixed-wing aircraft (e.g., \cite{nelson2007vector,lawrence2008lyapunov,rezende2018robust,yao2018robotic3d}). It has also been reported in \cite{sujit2014survey,caharija2015comparison} that VF methods can achieve good path-following accuracy with relatively low control effort. More recently, GVF methods have been extended to distributed cooperative path following of multi-robot systems \cite{hu2023spontaneous,yao2023distributed}, as well as to inverse-kinematics control of mobile robots \cite{zhou2025inverse}.

\begin{figure}[!t]
    \centering
    \subfloat[]{
        \includegraphics[width=0.43\linewidth]{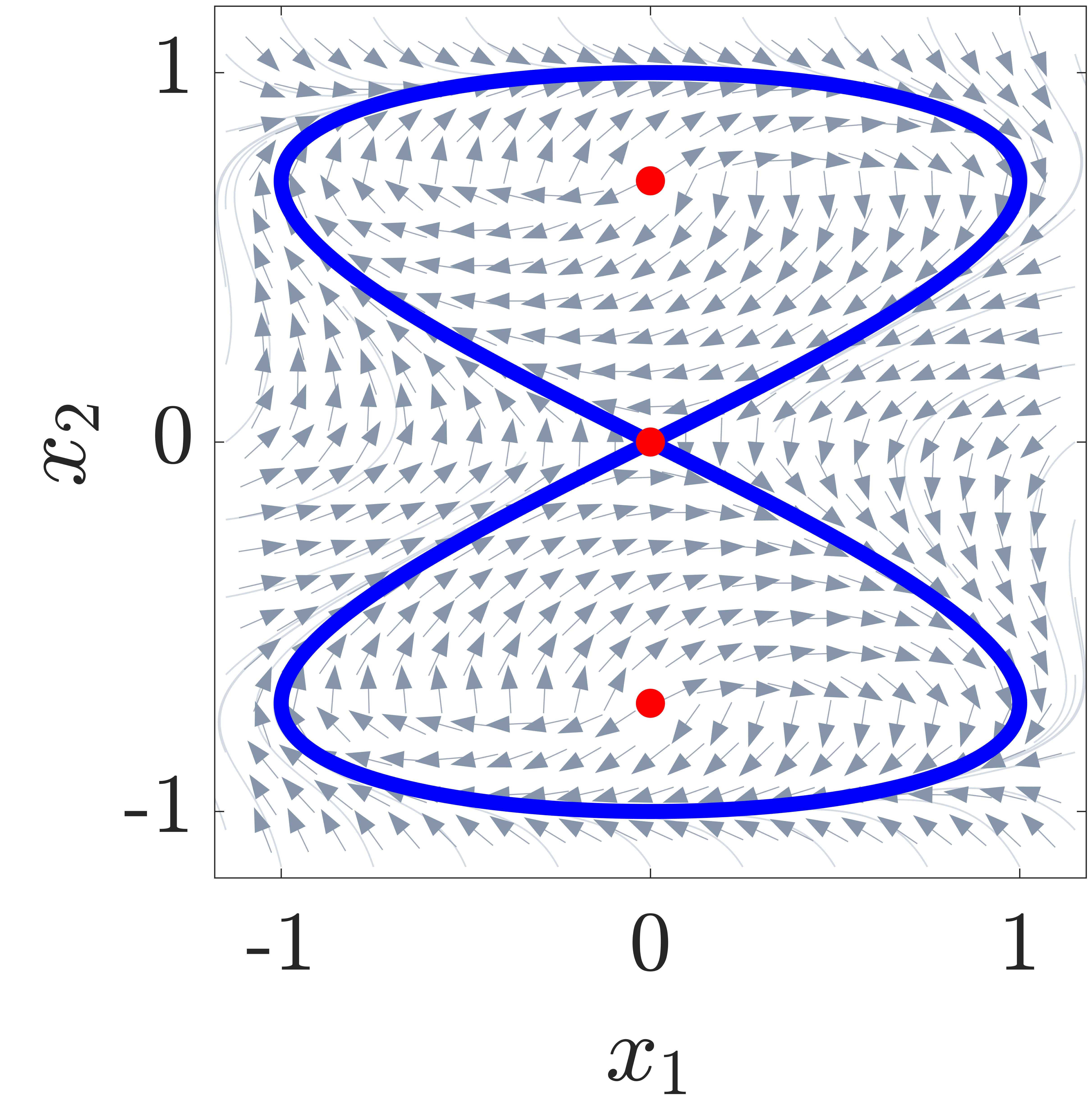}
        \label{fig:001a}
    }
    \subfloat[]{
        \includegraphics[width=0.43\linewidth]{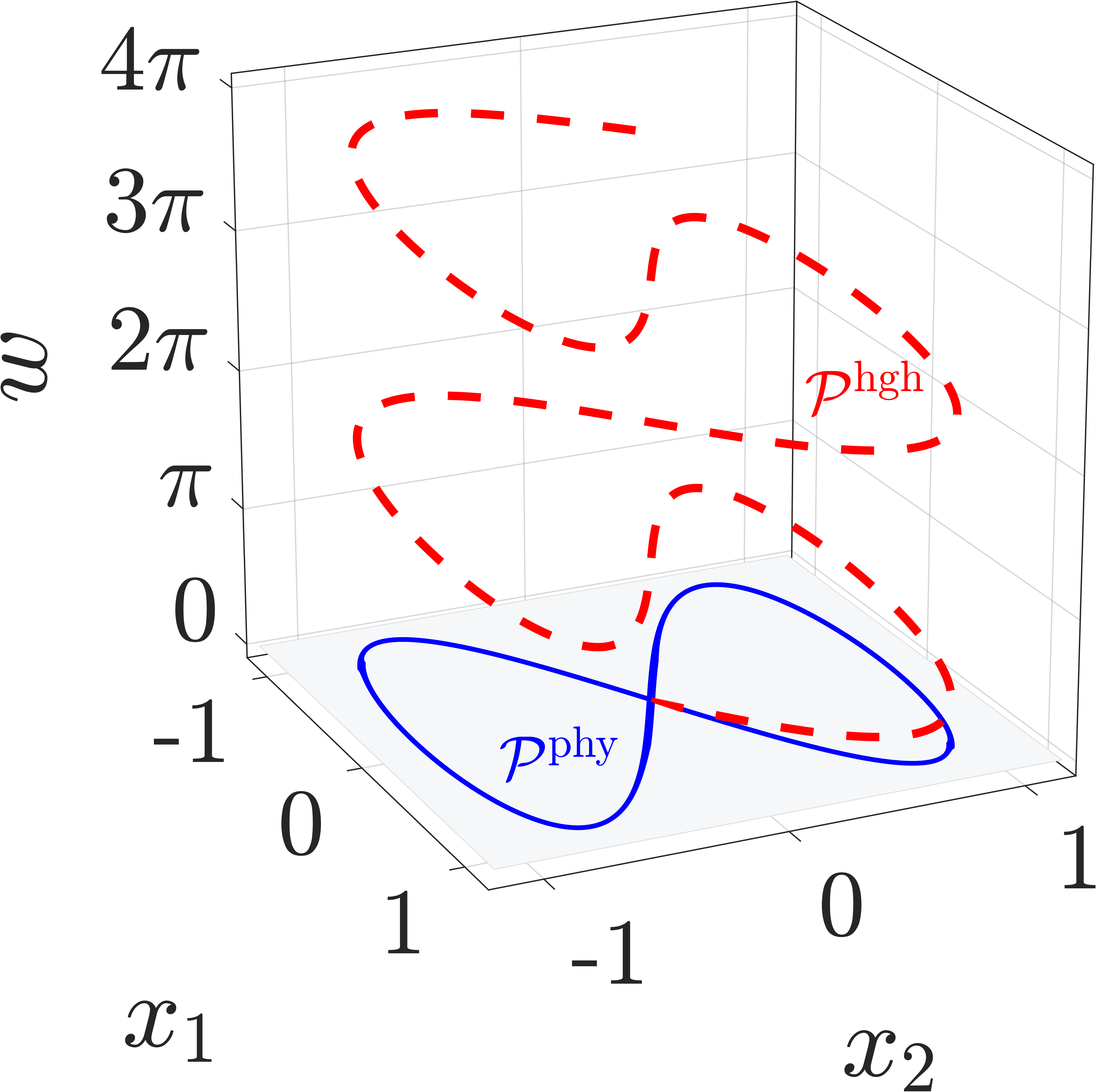}
        \label{fig:001b}
    }
    \caption{(a) The normalized vector field \cite{kapitanyuk2018gvf} for a figure “$\infty$” path described by $\phi(x_1,x_2)=x_1^2-4x_2^2+4x_2^4=0$. The three red points are the singular points (i.e., points where the vector field vanishes). (b) The 2D figure “$\infty$” path is “stretched” along the $w$-axis in the higher-dimensional space to form the higher-dimensional path.}
    \label{fig:001}
\end{figure}

However, a significant difficulty in the analysis and application of GVF methods arises when there are singular points at which the vector field becomes zero. In this case, the guidance direction becomes undefined when the robot reaches these points, and global convergence to the desired path cannot be guaranteed (e.g., \cite{kapitanyuk2018gvf,goncalves2010vector,yao2018robotic3d,yao2020path3d}). This problem is especially relevant for closed and self-intersecting paths. To address this problem, the work in \cite{yao2021sfgvf} treated the \textit{path parameter} as an additional virtual coordinate and “stretched” a closed or self-intersecting path in the physical space into an unbounded and non-self-intersecting path in a higher-dimensional space (see Fig.~\ref{fig:001}). An SF-GVF is then constructed in the higher-dimensional space, which enables global path following of the desired path. However, while this construction resolves the singularity issue associated with geometric path following, it introduces a new issue: the magnitude of the vector field (aka, desired speed) depends on the chosen path parameterization. In particular, when the robot moves on the desired path following the direction and speed given by the vector field, the SF-GVF can only regulate the speed of the path parameter (i.e., the virtual coordinate), whereas the physical speed is determined by the norm of the tangent vector of the parameterized path and cannot be independently controlled. 

As a result, different parameterizations of the same geometric path may lead to different physical speeds. In other words, the SF-GVF guarantees path following in the geometric sense, but the actual speed along the path remains \textit{coupled} with the path parameterization. For a robot with bounded acceleration, this dependence can be detrimental: an excessively large physical speed in a high-curvature segment may require a normal acceleration beyond the actuator limit, resulting in acceleration saturation and degraded path-following accuracy. To remove this dependence, partially normalizing the physical projection of the SF-GVF was suggested in \cite{yao2021sfgvf}. However, since the norm of the physical projection appears in the denominator and may vanish, the resulting partially normalized vector field is not guaranteed to be globally well-defined over the higher-dimensional space (see Section~\ref{sec:002}).


\textbf{Contributions:} To address the issue mentioned above, we propose a new SF-GVF for robot path following and design a corresponding saturated controller for second-order kinematic models with acceleration constraints. In particular, we redesign the converging and propagation terms and introduce a \textbf{P}rescribed \textbf{P}hysical \textbf{S}peed (PPS) to decouple the physical speed from the path parameterization. There are many appealing features of our approach:  1) The new GVF is globally well-defined and has no off-path singular points; for a positive PPS, it is singularity-free; 2) It guarantees that the path-following and speed-regulation control objectives are achieved globally regardless of the robot's initial position;  3) For different parameterizations of the same geometric path, the path-error dynamics of the vector field remain identical. The approach has been verified by high-speed flight experiments with a quadrotor. We further design a curvature-aware PPS for path-traversal tasks to mitigate path-following degradation caused by excessive acceleration demand.


\section{Preliminaries}\label{sec:002}

The SF-GVF with an additional path-parameter coordinate guarantees global convergence to the desired path and enables self-intersecting path following \cite{yao2021sfgvf}. Suppose an $n$-D physical desired path $\mathcal{P}^{\mathrm{phy}}$ is parameterized by $\boldsymbol{x}= \boldsymbol{f}(w)$, where $\boldsymbol{x}=\begin{bmatrix} x_1, x_2,\dots, x_n \end{bmatrix}^\top$ is the n-dimensional coordinate, $\boldsymbol{f}(w)=\begin{bmatrix} f_1(w), f_2(w),\dots, f_n(w) \end{bmatrix}^\top$, $w\in\mathbb{R}$ is the path parameter, and $\boldsymbol{f}\in C^2$. For $i=1,2,\dots,n$, define $\phi_i(\boldsymbol{\xi})=x_i - f_i(w)$, where $\boldsymbol{\xi}=[\boldsymbol{x}^\top,w]^\top\in \mathbb{R}^{n+1}$ is the higher-dimensional coordinate. The corresponding higher-dimensional desired path is the intersection of the $n$ hypersurfaces $\phi_i(\boldsymbol{\xi})=0$ \cite{goncalves2010vector,yao2018robotic3d,hu2023spontaneous}; that is, $\mathcal{P}^{\mathrm{hgh}}=\left\{\boldsymbol{\xi}\in\mathbb{R}^{n+1}\mid \phi_i(\boldsymbol{\xi})= 0,i=1,2,\dots,n\right\}$. Its projection onto the physical coordinates is exactly $\mathcal{P}^{\mathrm{phy}}$ (see Fig.~\ref{fig:001b}). Thus, a higher-dimensional GVF corresponding to $\mathcal{P}^{\mathrm{hgh}}$ can be projected onto the $n$-D physical space to guide the robot along $\mathcal{P}^{\mathrm{phy}}$. In \cite{yao2021sfgvf}, the SF-GVF is defined as
\begin{equation}\label{eq:high_vector_field}
\boldsymbol{\chi}^{\mathrm{hgh}}(\boldsymbol{\xi})=\times
\begin{pmatrix}
\nabla\phi_1,\ldots,\nabla\phi_n 
\end{pmatrix}-\sum_{i=1}^nk_i\phi_i\nabla\phi_i,
\end{equation}
where $k_i > 0$, and $\nabla\phi_i=\begin{bmatrix}0,\dots,1,\dots,-f_i^{\prime}(w)\end{bmatrix}^\top \in \mathbb{R}^{n+1}$ is the gradient of $\phi_i$ with respect to the higher-dimensional coordinate $\boldsymbol{\xi}$, with 1 being its $i$-th component and $f_i^{\prime}(w)=\mathrm{d}f_i(w)/\mathrm{d}w$. The operator $\times(\cdot)$ denotes the generalized cross product \cite[Ch.~7.2]{galbis2012vector}. Hence,
\begin{equation}\label{eq:ger_cross_product}
\times(\nabla\phi_1,\ldots,\nabla\phi_n)=(-1)^n
\begin{bmatrix}
\boldsymbol{f}^{\prime}(w) \\
1
\end{bmatrix}\in\mathbb{R}^{n+1}.
\end{equation}
For convenience, denote
$\nabla_{\times}\boldsymbol{\phi}
:=\times(\nabla\phi_1,\ldots,\nabla\phi_n)$. The physical interpretation of the vector field $\boldsymbol{\chi}^{\mathrm{hgh}}$ is clear. The first term $\nabla_{\times}\boldsymbol{\phi}$ is orthogonal to all the gradient vectors \cite[Proposition~7.2.1]{galbis2012vector} and thus provides a propagation direction along the desired path. The second term $-\sum_{i=1}^{n}k_i\phi_i\nabla\phi_i$ is a weighted sum of the gradient vectors and guides the robot to converge to the desired path. It is worth noting that the $(n+1)$-th component of $\nabla_{\times}\boldsymbol{\phi}$ is the constant $(-1)^n$, which is independent of the particular parameterization of the desired path. Therefore, $\|\nabla_\times\phi(\boldsymbol{\xi})\|\neq0$ for all $\boldsymbol{\xi}\in\mathbb{R}^{n+1}$. Together with the orthogonality between the propagation and converging terms, this implies that $\boldsymbol{\chi}^{\mathrm{hgh}}(\boldsymbol{\xi}) \neq \boldsymbol{0}, \forall \boldsymbol{\xi} \in \mathbb{R}^{n+1}$, and hence there are no singular points in the higher-dimensional space $\mathbb{R}^{n+1}$. We rewrite \eqref{eq:high_vector_field} in the following compact form:
\begin{equation}\label{eq:form_compact}
\boldsymbol{\chi}^{\mathrm{hgh}}(\boldsymbol{\xi}) = \begin{bmatrix}
\boldsymbol{\chi}_p(\boldsymbol{\xi}) \\ \chi_w(\boldsymbol{\xi})
\end{bmatrix} = \begin{bmatrix} (-1)^n\boldsymbol{f}^{\prime}(w)-\boldsymbol{K\phi} \\
(-1)^n+\boldsymbol{f}^{\prime}(w)^\top\boldsymbol{K\phi} \end{bmatrix},
\end{equation}
where $\boldsymbol{K}=\operatorname{diag}(k_1,\ldots,k_n)\in \mathbb{R}^{n \times n}$ is a diagonal matrix and $\boldsymbol{\phi} = \begin{bmatrix}\phi_1,\phi_2,\dots,\phi_n\end{bmatrix}^\top$. Consider the autonomous system 
\begin{equation}\label{eq:ode}
\dot{\boldsymbol{\xi}}=\begin{bmatrix}
\dot{\boldsymbol{x}} \\ \dot{w}
\end{bmatrix} = \begin{bmatrix}
\boldsymbol{\chi}_p(\boldsymbol{\xi}) \\ \chi_w(\boldsymbol{\xi})
\end{bmatrix}=\boldsymbol{\chi}^{\mathrm{hgh}}(\boldsymbol{\xi}).
\end{equation}
When $\|\boldsymbol{\phi}\| = 0$, one has $\boldsymbol{\xi} \in \mathcal{P}^{\mathrm{hgh}}$ and consequently $\boldsymbol{x}\in \mathcal{P}^{\mathrm{phy}}$. In this case, the speed of the path parameter is $\dot{w} = (-1)^n$, while the physical coordinates satisfy $\|\dot{\boldsymbol{x}}\|=\|(-1)^n\boldsymbol{f}^{\prime}(w)\|=\|\boldsymbol{f}^{\prime}(w)\|$. Therefore, the physical speed depends on the chosen path parameterization and cannot be specified independently. To overcome the aforementioned dependence of the physical speed on the path parameterization, a natural idea is to partially normalize the higher-dimensional GVF with respect to its physical projection; that is,
\begin{equation}\label{eq:partone}
\widehat{\boldsymbol{\chi}}^\mathrm{hgh}(\boldsymbol{\xi})=\frac{\boldsymbol{\chi}^\mathrm{hgh}(\boldsymbol{\xi})}{\|\boldsymbol{\chi}_p(\boldsymbol{\xi})\|}.
\end{equation}
This idea is also mentioned in \cite[Sec.~\Romannum{8}]{yao2021sfgvf}. However, the denominator $\|\boldsymbol{\chi}_p(\xi)\|$ may be zero and \eqref{eq:partone} is not well defined in the normalization. In fact, the vector field $\widehat{\boldsymbol{\chi}}^\mathrm{hgh}$ is not well-defined on the set $\mathcal{Z}=\left\{\boldsymbol{\xi}\in\mathbb{R}^{n+1}\mid \|\boldsymbol{\chi}_p(\boldsymbol{\xi})\|=0\right\}$, called the \emph{zero set}. From \eqref{eq:form_compact}, the zero set is calculated as follows:
\begin{equation*}
\mathcal{Z}=\left\{\boldsymbol{\xi}\in \mathbb{R}^{n+1}\mid\boldsymbol{x}=\boldsymbol{f}(w)+(-1)^n\boldsymbol{K}^{-1}\boldsymbol{f}^{\prime}(w)\right\}.
\end{equation*}
For every $w\in\mathbb{R}$, there exists a unique physical coordinate $\boldsymbol{x}\in\mathbb{R}^{n}$ such that the corresponding higher-dimensional coordinate belongs to $\mathcal{Z}$; hence, $\mathcal Z$ is always nonempty. This implies an inherent limitation: the partial-normalization approach cannot define a globally well-defined vector field on the higher-dimensional space, no matter how one chooses the path parametrization. Two examples are shown in Fig.~\ref{fig:003}.

\begin{figure}[!t]
    \centering
    \subfloat[]{
        \includegraphics[width=0.43\linewidth]{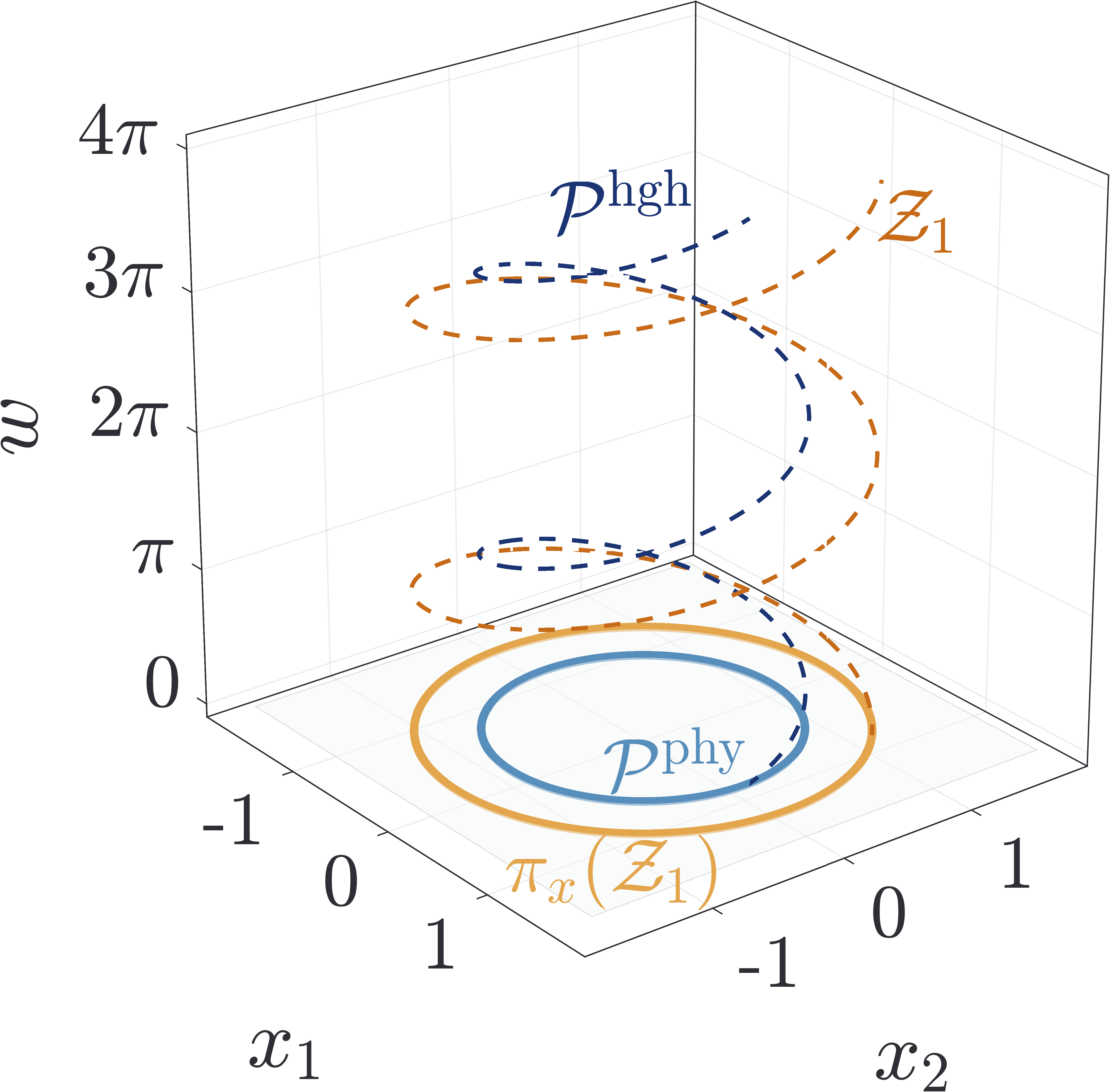}
        \label{fig:003a}
    }
    \subfloat[]{
        \includegraphics[width=0.43\linewidth]{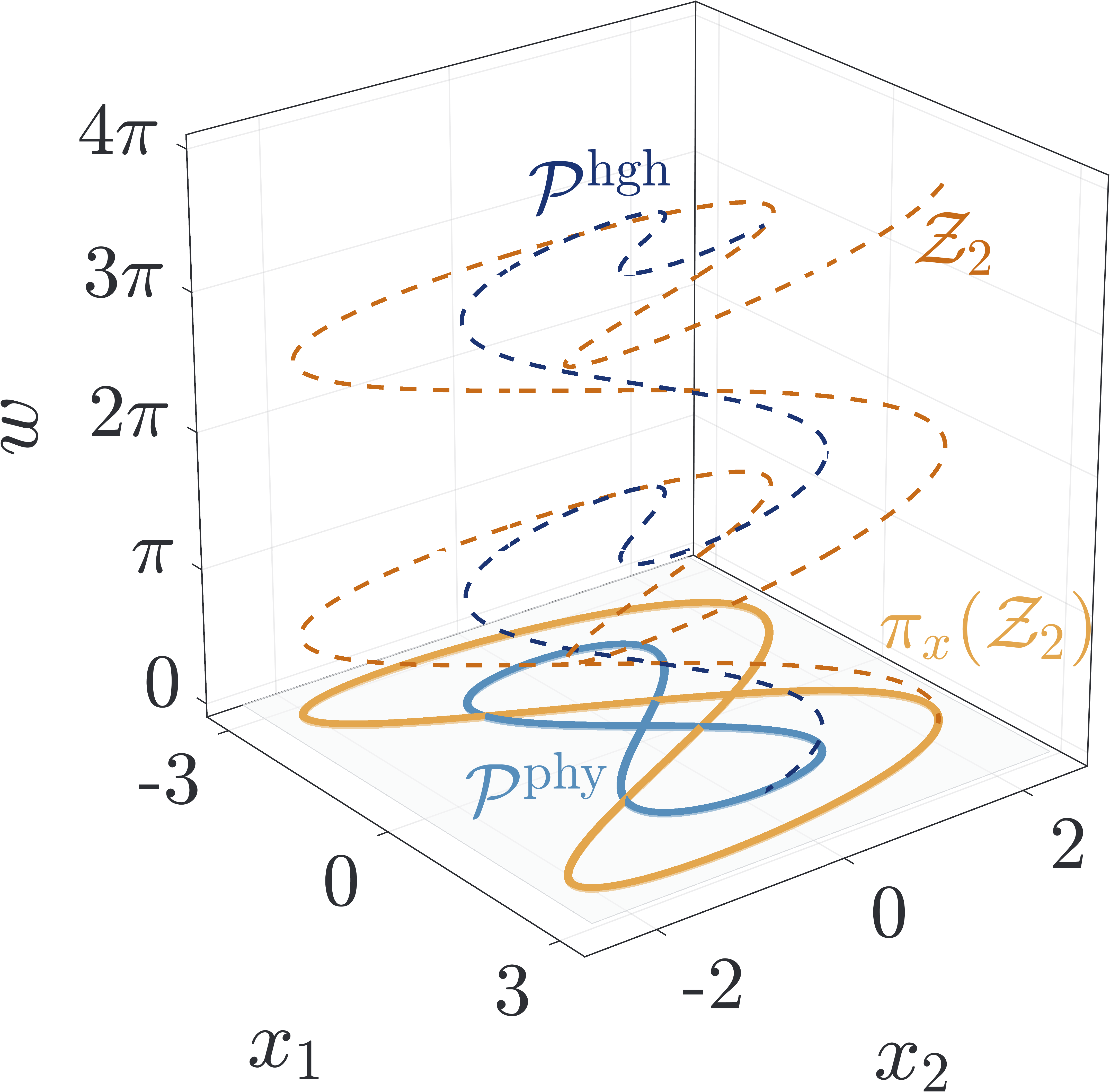}
        \label{fig:003b}
    }
    \caption{In both examples, $n=2$ and $\boldsymbol{K}=\boldsymbol{I}_2$.  (a) $\boldsymbol{f}(w)=[\cos w,\sin w]^\top$, with $\mathcal{Z}_1=\{[\sqrt{2}\cos\left(w+\frac{\pi}{4}\right) ,\sqrt{2}\sin\left(w+\frac{\pi}{4}\right) ,w]^\top \mid w\in\mathbb{R}\}$. (b) $\boldsymbol{f}(w)=[2\cos w,\sin 2w]^\top$, with $\mathcal{Z}_2=\{[2\sqrt{2}\cos\left(w+\frac{\pi}{4}\right) ,\sqrt{5}\sin\left(2w+\arctan2\right) ,w]^\top \mid w\in\mathbb{R}\}$. Orange dashed and solid curves denote the zero set $\mathcal{Z}_i$ and its projection $\pi_x(\mathcal{Z}_i)$ onto the plane $w=0$, respectively.}
    \label{fig:003}
\end{figure}


Since the partially normalized vector field is defined only on $\mathbb{R}^{n+1}\setminus\mathcal{Z}$, the differential equation
$\dot{\boldsymbol{\xi}}=\widehat{\boldsymbol{\chi}}^\mathrm{hgh}(\boldsymbol{\xi})$
is not well-defined when the initial condition satisfies $\boldsymbol{\xi}(0)\in\mathcal{Z}$. Moreover, as $\boldsymbol{\xi}(t)$ approaches $\mathcal{Z}$, one has $\|\boldsymbol{\chi}_p(\boldsymbol{\xi})\|\to0$, and hence $\|\widehat{\boldsymbol{\chi}}^\mathrm{hgh}(\boldsymbol{\xi})\|$ becomes unbounded. Therefore, this approach cannot guarantee a globally well-defined solution for arbitrary initial conditions, and consequently cannot simultaneously guarantee global convergence to the desired path and physical-speed regulation. Based on the above discussion, we focus on the following question:

\textbf{Question:} Is it possible to regulate the physical speed of the robot independently of the path parameterization, while retaining the global well-definedness and global convergence properties?


\section{Design of Singularity-Free Guiding Vector Field with Prescribed Physical Speed}\label{sec:003}

In this section, we formulate the vector-field guided path-following (VF-PF) problem with prescribed physical speed (PPS) and construct the corresponding SF-GVF. We first formally define the PPS and the VF-PF problem as follows.

\begin{definition}[Prescribed physical speed]\label{def:001}
    A PPS is a continuously differentiable function $v_d:\mathbb{R}\to\mathbb{R}_{\geq0}$ satisfying $0\leq v_d(w)\leq v_+,|v_d^\prime(w)|\leq L_v,\forall w\in\mathbb{R}$ for some constants $v_+>0$ and $L_v \geq 0$. The value $v_d(w)$ specifies the desired physical speed of the robot at the path parameter $w$. Here, $w(t)$ is a virtual state and $w(0)$ can be initialized arbitrarily even when \(\boldsymbol{x}(0)\notin\mathcal P^{\mathrm{phy}}\).
\end{definition}

\begin{problem}[VF-PF problem with PPS] Given a physical desired path $\mathcal{P}^{\mathrm{phy}}=\left\{\boldsymbol{x}\in\mathbb{R}^n\mid \boldsymbol{x} = \boldsymbol{f}(w),w\in\mathbb{R}\right\}$, its corresponding higher-dimensional desired path $\mathcal{P}^{\mathrm{hgh}} = \{\boldsymbol{\xi} \in \mathbb{R}^{n+1} \mid \boldsymbol{\phi}(\boldsymbol{\xi}) = \mathbf{0},\boldsymbol{\phi} = \boldsymbol{x} - \boldsymbol{f}(w)\}$, and a PPS
$v_d(w)$, the VF-PF problem with PPS is to design a continuously differentiable higher-dimensional GVF $\widetilde{\boldsymbol{\chi}}^{\mathrm{hgh}}:\mathbb{R}^{n+1}\to\mathbb{R}^{n+1}$ for the autonomous system $\dot{\boldsymbol{\xi}}(t)=\widetilde{\boldsymbol{\chi}}^{\mathrm{hgh}}(\boldsymbol{\xi}(t))$ such that the following three conditions are satisfied:
\begin{itemize}
\item[1)] \textbf{(Well-definedness)}
The vector field $\widetilde{\boldsymbol{\chi}}^{\mathrm{hgh}}$ is globally well-defined on the higher-dimensional space $\mathbb{R}^{n+1}$.
\item[2)] \textbf{(Path following)}
If a trajectory starts from the higher-dimensional desired path, i.e.,
$\boldsymbol{\xi}(0)\in\mathcal{P}^{\mathrm{hgh}}$, then it stays on the path for all $t\geq0$, i.e.,
$\boldsymbol{\xi}(t)\in\mathcal{P}^{\mathrm{hgh}}$ for all $t\geq0$.
Moreover, for an arbitrary initial condition
$\boldsymbol{\xi}(0)\in\mathbb{R}^{n+1}$, the trajectory asymptotically converges to $\mathcal{P}^{\mathrm{hgh}}$ as $t\to\infty$; that is, $\lim_{t\to\infty} \|\boldsymbol{\phi}(\boldsymbol{\xi}(t))\|=0$.
\item[3)] \textbf{(PPS regulation)}
As the trajectory converges to the desired path, i.e.,
$\|\boldsymbol{\phi}(\boldsymbol{\xi}(t))\|\to0$,
its physical speed $\|\dot{\boldsymbol{x}}(t)\|$ asymptotically converges to the prescribed physical speed $v_d(w(t))$; that is,
$\lim_{t\to\infty} \big| \|\dot{\boldsymbol{x}}(t)\| - v_{d}(w(t)) \big| = 0$.
In addition, when the trajectory lies on the higher-dimensional desired path $\mathcal{P}^{\mathrm{hgh}}$, i.e.,
$\boldsymbol{\phi}(\boldsymbol{\xi}(t))=\mathbf{0}$,
it holds that $\|\dot{\boldsymbol{x}}(t)\| = v_{d}(w(t))$.
\end{itemize}
\end{problem}

In addition, we impose the following mild regularity assumption on the parameterization of the desired path.

\begin{assumption}\label{ass:001}
There exist positive constants $\gamma_1$, $\gamma_2$, and $\gamma_3$ such that $0<\gamma_1\leq\Vert\boldsymbol{f}^\prime(w)\Vert\leq\gamma_2,\Vert\boldsymbol{f}^{\prime\prime}(w)\Vert\leq\gamma_3$ for all $w\in\mathbb{R}$.
\end{assumption}

\begin{remark}
Unlike trajectory tracking, the PPS $v_d(w(t))$ is specified with respect to
the evolving virtual state $w(t)$ rather than an exogenous time variable.
Hence, no prescribed time schedule is imposed on the robot's progress along
the desired path. Assumption~\ref{ass:001} imposes regularity conditions on the path parameterization rather than geometric restrictions on the desired path. For example, this assumption is satisfied by any \(C^2\) periodic regular curve, as well as by standard regular parameterizations of straight lines and helices. It is worth noting that the lower bound on $\|\boldsymbol{f}'(w)\|$ does not eliminate the singularities caused by partial normalization, since $\|\boldsymbol{\chi}_p\|$ may still vanish. The PPS $v_d$ can be specified according to different task requirements. For example, $v_d\equiv c$ yields a constant speed, whereas a path-dependent $v_d$ can specify spatially varying or curvature-dependent speeds. In particular, $v_d(w)=0$ is allowed, corresponding to a prescribed stop on the desired path.
\hfill $\triangleleft$
\end{remark}


We propose the following vector field:
\begin{equation}\label{eq:proposed_vector_field}
\widetilde{\boldsymbol{\chi}}^{\mathrm{hgh}}
(\boldsymbol{\xi})
=
\sigma\frac{v_d(w)}
{\|\boldsymbol{f}^{\prime}(w)\|}
\nabla_{\times}\boldsymbol{\phi}
-
\boldsymbol{M}(w)
\sum_{i=1}^{n}
k_i\phi_i\nabla\phi_i,
\end{equation}
where $\boldsymbol{M}(w)= \begin{bmatrix} \boldsymbol{I}_n & \boldsymbol{0}_{n\times1} \\ \boldsymbol{0}_{n\times1}^\top & \|\boldsymbol{f}^{\prime}(w)\|^{-2} \end{bmatrix}\in \mathbb{R}^{(n+1) \times (n+1)}$ and $\sigma=(-1)^n$ is used to compensate for the sign introduced by the order of the generalized cross product in \eqref{eq:ger_cross_product}, such that a positive $v_d$ corresponds to the direction of increasing $w$. The first term provides a propagation direction whose physical component has magnitude $v_d(w)$, while the second term provides the converging direction with different scalings on the physical and virtual components. Specifically, the matrix $\boldsymbol{M}(w)$ leaves the physical component of $\widetilde{\boldsymbol{\chi}}^{\mathrm{hgh}}$, i.e., the vector consisting of its first $n$ components, unchanged, while scaling the virtual component, i.e., the $(n+1)$-th component, by $\|\boldsymbol{f}^{\prime}(w)\|^{-2}$. More explicitly, \eqref{eq:proposed_vector_field} can be compactly written as
\begin{equation}\label{eq:components_concise}
\scalebox{1.0}{$\displaystyle
\widetilde{\boldsymbol{\chi}}^{\mathrm{hgh}}(\boldsymbol{\xi}) =
\begin{bmatrix}
\widetilde{\boldsymbol{\chi}}_{p}(\boldsymbol{\xi}) \\
\widetilde{\chi}_{w}(\boldsymbol{\xi})
\end{bmatrix}
=
\begin{bmatrix}
\dfrac{v_d(w)}{\|\boldsymbol{f}^{\prime}(w)\|}
\boldsymbol{f}^{\prime}(w) - \boldsymbol{K\phi} \\
\dfrac{v_d(w)}{\|\boldsymbol{f}^{\prime}(w)\|}
+
\dfrac{\boldsymbol{f}^{\prime}(w)^\top\boldsymbol{K\phi}}
{\|\boldsymbol{f}^{\prime}(w)\|^2}
\end{bmatrix}
$}.
\end{equation}

We will show that the vector field $\widetilde{\boldsymbol{\chi}}^{\mathrm{hgh}}$ is globally well-defined and globally convergent to the desired path.



\begin{lemma}
The vector field $\widetilde{\boldsymbol{\chi}}^{\mathrm{hgh}}$
is continuously differentiable and globally well-defined on
$\mathbb{R}^{n+1}$. Moreover, $\widetilde{\boldsymbol{\chi}}^{\mathrm{hgh}}(\boldsymbol{\xi})
=\mathbf{0}$ if and only if $\boldsymbol{\phi}(\boldsymbol{\xi})=\mathbf{0},v_d(w)=0$. Therefore, the proposed vector field has no off-path singular points. In particular, if $v_d(w)>0$ for all $w\in\mathbb{R}$, the vector field is singularity-free.
\end{lemma}

\begin{proof}
It follows from Assumption \ref{ass:001} that
$\Vert{}\boldsymbol{f}^{\prime}(w)\Vert{} \geq \gamma_1 > 0$.
Therefore, $1/\|\boldsymbol{f}^{\prime}(w)\|$ is well-defined for all $w\in\mathbb{R}$, and hence the vector field is globally well-defined on $\mathbb{R}^{n+1}$. Since $\boldsymbol{f}\in C^2$ and $v_d\in C^1$, it is clear that $\widetilde{\boldsymbol{\chi}}^{\mathrm{hgh}}(\boldsymbol{\xi})$ is continuously differentiable on $\mathbb{R}^{n+1}$. To characterize the set of singular points, let that
$\widetilde{\boldsymbol{\chi}}^{\mathrm{hgh}}(\boldsymbol{\xi})
=\mathbf{0}$.
Substituting the first $n$ components into the virtual component in \eqref{eq:components_concise} yields $0=\frac{2v_d(w)}{\|\boldsymbol{f}'(w)\|}$, and hence $\boldsymbol{\phi}=\mathbf{0}$  and $v_d(w)=0$.
\end{proof}


We next show its global path-following and PPS-regulation properties.

\begin{theorem}\label{thm:main_convergence}
Given a PPS $v_d(w)$ and an $n$-D desired path $\mathcal{P}^{\mathrm{phy}}$ with a path parametrization $\boldsymbol{f}: \mathbb{R}\to \mathbb{R}^n$ satisfying Assumption~\ref{ass:001}, for any initial condition $\boldsymbol{\xi}(0)\in\mathbb{R}^{n+1}$, the solution
$\boldsymbol{\xi}(t)=[\boldsymbol{x}(t)^\top,w(t)]^\top$
to the autonomous system
$\dot{\boldsymbol{\xi}}=\widetilde{\boldsymbol{\chi}}^{\mathrm{hgh}}(\boldsymbol{\xi})$
exists for all $t\geq0$. Moreover, both the path error $\|\boldsymbol{\phi}(t)\|$ and the physical-speed error
$\big| \|\dot{\boldsymbol{x}}(t)\| - v_{d}(w(t)) \big|$
converge globally exponentially to zero.
\end{theorem}

\begin{proof}
The implicit path error is given by
$\boldsymbol{\phi}(\boldsymbol{\xi})=\boldsymbol{x}-\boldsymbol{f}(w)$.
Taking its time derivative and substituting \eqref{eq:components_concise} yields
\begin{align}\label{eq:error}
    \dot{\boldsymbol{\phi}} & =\dot{\boldsymbol{x}}-\boldsymbol{f}^{\prime}(w)\dot{w}  =-\left(\boldsymbol{I}_n+\frac{\boldsymbol{f}^{\prime}(w)\boldsymbol{f}^{\prime}(w)^\top}{\|\boldsymbol{f}^{\prime}(w)\|^2}\right)\boldsymbol{K}\boldsymbol{\phi}.
    \end{align}

Consider the positive definite Lyapunov candidate function $V(\boldsymbol{\phi}) = \frac{1}{2} \boldsymbol{\phi}^\top \boldsymbol{K} \boldsymbol{\phi}$. Let $\lambda_{\min}$ and $\lambda_{\max}$ denote the minimum and maximum eigenvalues of the positive definite matrix $\boldsymbol{K}$, respectively. Then,
\begin{equation}\label{eq:V_bounds}
        \frac{\lambda_{\min}}{2} \|\boldsymbol{\phi}\|^2 \leq V(\boldsymbol{\phi}) \leq \frac{\lambda_{\max}}{2} \|\boldsymbol{\phi}\|^2.
    \end{equation}
Taking the time derivative of $V$ gives
\begin{align} \label{eq:Lyapunov}
    \dot{V}(\boldsymbol{\phi}) &= \boldsymbol{\phi}^\top \boldsymbol{K} \dot{\boldsymbol{\phi}}= -\|\boldsymbol{K}\boldsymbol{\phi}\|^2 - \frac{\left( \boldsymbol{f}'(w)^\top \boldsymbol{K}\boldsymbol{\phi} \right)^2}{\|\boldsymbol{f}'(w)\|^2}  \notag\\
    &\leq -\|\boldsymbol{K}\boldsymbol{\phi}\|^2\leq -\lambda_{\min}^2 \|\boldsymbol{\phi}\|^2. 
    \end{align}
    Combining \eqref{eq:Lyapunov} with \eqref{eq:V_bounds} gives
    $\dot{V}\leq-cV$, where
    $c=2\lambda_{\min}^2/\lambda_{\max}>0$.
    Therefore,
    $V(t)\leq e^{-ct}V(0)$.
    Using \eqref{eq:V_bounds} again, we obtain $\|\boldsymbol{\phi}(t)\| \leq \sqrt{\lambda_{\max}/\lambda_{\min}} e^{-\frac{c}{2}t} \|\boldsymbol{\phi}(0)\|$. Hence, by the comparison lemma \cite[Lemma~3.4]{khalil2002nonlinear}, $\|\boldsymbol{\phi}(t)\|$ converges globally exponentially to zero. Next, we show that the solution exists for all $t\geq0$. From \eqref{eq:components_concise}, the virtual-coordinate dynamic is $\dot{w} = \frac{v_d(w)}{\|\boldsymbol{f}'(w)\|} + \frac{\boldsymbol{f}'(w)^\top \boldsymbol{K}\boldsymbol{\phi}}{\|\boldsymbol{f}'(w)\|^2}$. By Definition~\ref{def:001} and Assumption~\ref{ass:001}, $v_d(w)$ is bounded and $\|\boldsymbol{f}'(w)\|$ is bounded away from zero. Since $\|\boldsymbol{\phi}(t)\|$ decays exponentially and is therefore bounded, it follows that $\dot{w}(t)$ is bounded on every finite time interval. Therefore, $w(t)$ cannot escape to infinity in finite time. Similarly, since $\boldsymbol{x}(t) = \boldsymbol{f}(w(t)) + \boldsymbol{\phi}(t)$ and $\boldsymbol{f}$ is continuous, $\boldsymbol{x}(t)$ cannot escape to infinity in finite time either. Hence, the solution exists for all $t\geq0$. Finally, from \eqref{eq:components_concise} and the reverse triangle inequality $\big| \|\boldsymbol{a}\| - \|\boldsymbol{b}\| \big| \leq \|\boldsymbol{a} - \boldsymbol{b}\|$, the physical-speed error satisfies
    \begin{align*}
\scalebox{1.0}{$\displaystyle
\left| \|\dot{\boldsymbol{x}}(t)\| - v_d(w(t)) \right|
$}
&\leq
\scalebox{1.0}{$\displaystyle
\left\|
\dot{\boldsymbol{x}}(t)
-
\frac{v_d(w(t))}{\|\boldsymbol{f}'(w(t))\|}
\boldsymbol{f}'(w(t))
\right\|
$}
\\
&=
\scalebox{1.0}{$\displaystyle
\|\boldsymbol{K}\boldsymbol{\phi}(t)\|
\leq
\|\boldsymbol{K}\| \|\boldsymbol{\phi}(t)\|
$}.
\end{align*}
Therefore, the physical-speed error also converges globally exponentially to zero.
\end{proof}

The path-error dynamics in \eqref{eq:error} further lead to the following reparameterization-invariance property.

\begin{corollary}[\textit{Reparameterization-invariant path-error dynamics}]\label{cor:001}
Let $\boldsymbol{f}:\mathbb{R}\to\mathbb{R}^n$ satisfy Assumption~\ref{ass:001}, and let
$h:\mathbb{R}\to\mathbb{R}$ be a $\mathcal{C}^2$ diffeomorphism satisfying
$h'(s)\neq0$ for all $s\in\mathbb{R}$.
Define the reparameterized curve
$\boldsymbol{f}_h(s)=\boldsymbol{f}(h(s))$
and suppose that $\boldsymbol{f}_h$ also satisfies Assumption~\ref{ass:001}.
For the same physical position
$\boldsymbol{x}\in\mathbb{R}^n$
and the corresponding parameters satisfying
$w=h(s)$, define $\boldsymbol{\phi}(\boldsymbol{x}, w) = \boldsymbol{x} - \boldsymbol{f}(w), \boldsymbol{\phi}_h(\boldsymbol{x}, s)= \boldsymbol{x} - \boldsymbol{f}_h(s)$. Then $\boldsymbol{\phi}_h(\boldsymbol{x}, s) = \boldsymbol{\phi}(\boldsymbol{x}, w)$. Moreover, the path-error dynamics in \eqref{eq:error} remain identical under the two parameterizations.
\end{corollary}

\begin{proof}
By the chain rule, $\boldsymbol{f}_h'(s) = \boldsymbol{f}'(h(s))h'(s)$.
Since $h'(s)\neq0$ and $\boldsymbol{f}'(h(s))\neq\boldsymbol{0}$, normalization gives
\begin{equation}\label{eq:cor1}
            \frac{\boldsymbol{f}_h'(s)}          {\Vert{}\boldsymbol{f}_h'(s)\Vert{}} =   \frac{\boldsymbol{f}'(h(s))h'(s)}        {\Vert{}\boldsymbol{f}'(h(s))h'(s)\Vert{}} =   \operatorname{sgn}(h'(s))\frac{\boldsymbol{f}'(w)}          {\Vert{}\boldsymbol{f}'(w)\Vert{}}. 
    \end{equation}
Here, $\operatorname{sgn}(\cdot)$ denotes the sign function. On the other hand, when $w=h(s)$, $\boldsymbol{\phi}_h(\boldsymbol{x}, s) = \boldsymbol{x} - \boldsymbol{f}_h(s) = \boldsymbol{x} - \boldsymbol{f}(h(s)) = \boldsymbol{x} - \boldsymbol{f}(w) = \boldsymbol{\phi}(\boldsymbol{x}, w)$. Substituting \eqref{eq:cor1} into the right-hand side of \eqref{eq:error}, the factor $\operatorname{sgn}(h'(s))$ cancels in the normalized outer product. Therefore, the path-error dynamics remain unchanged.
\end{proof}

\begin{remark}
Corollary~\ref{cor:001} establishes only that the path-error dynamics are invariant under regular reparameterizations; it does not imply that the complete higher-dimensional vector field or the propagation direction along the path remains unchanged. In particular, when $h'(s)<0$, the reparameterization reverses the direction in which the path parameter evolves along the curve. By contrast, for an orientation-preserving reparameterization with $h'(s)>0$, if the PPS is correspondingly reparameterized as $\widehat v_d(s)=v_d(h(s))$ and the initial path parameters satisfy $w(0)=h(s(0))$, then the resulting physical trajectories coincide.\hfill $\triangleleft$
\end{remark}

\section{Controller for Second-Order Kinematic Models}
If the robot motion can be approximated by a single-integrator model, then the physical vector field $\widetilde{\boldsymbol{\chi}}_{p}$ in \eqref{eq:components_concise} can be directly used as the control input of the robot (e.g., \cite{goncalves2010vector,9785912}). However, for many robotic systems, the control inputs correspond to accelerations or forces. In this section, we design a control law for such a model. Consider the following second-order dynamics with acceleration saturation:
\begin{equation}\label{eq:Second-Order}
    \dot{\boldsymbol{p}}=\boldsymbol{v}, \quad
    \dot{\boldsymbol{v}}=\boldsymbol{a}, \quad
    \|\boldsymbol{a}\|\leq \bar{a},
\end{equation}
where $\boldsymbol{p}\in\mathbb{R}^{n}$ and $\boldsymbol{v}\in\mathbb{R}^{n}$ denote the position and velocity of the robot, respectively, $\boldsymbol{a}\in\mathbb{R}^{n}$ is the acceleration control input, and $\bar{a}>0$ is the maximum allowable acceleration. We further introduce a reference-speed limit $\bar v$ satisfying $\bar v>\sup_w v_d(w)$. Since the SF-GVF with PPS provides a desired velocity, including both its magnitude and direction, to guide the robot motion, the key to the control algorithm is to make the actual velocity $\boldsymbol{v}$ track the physical vector field $\widetilde{\boldsymbol{\chi}}_{p}$. We first define the radial saturation function $\operatorname{Sat}_{c}:\mathbb{R}^{n}\to\mathbb{R}^{n}$ by
$\operatorname{Sat}_{c}(\boldsymbol{y})=\boldsymbol{y}$ if $\|\boldsymbol{y}\|\leq c$, and
$\operatorname{Sat}_{c}(\boldsymbol{y})=c\boldsymbol{y}/\|\boldsymbol{y}\|$ if $\|\boldsymbol{y}\|>c$, where $c>0$ is a given constant. The radial saturation function $\operatorname{Sat}_{c}$ is Lipschitz continuous. For convenience, we refer to a time interval during which $\|\boldsymbol{y}\|>c$ as a saturation period. Next, observe from \eqref{eq:components_concise} that the converging term increases linearly with the path error $\|\boldsymbol{\phi}\|$. Therefore, if the initial path error is large, $\|\widetilde{\boldsymbol{\chi}}_{p}\|$ may exceed the reference-speed limit $\bar v$, resulting in a physically infeasible velocity. To address this issue, we introduce a positive scalar scaling function of $\|\widetilde{\boldsymbol{\chi}}_{p}\|$ for the vector field $\widetilde{\boldsymbol{\chi}}^{\mathrm{hgh}}$, defined as
\begin{equation}\label{eq:suofang}
\scalebox{1.0}{$\displaystyle
\mu(x)=
\begin{cases}
1, & 0\leq x\leq v_s, \\
\dfrac{v_s+(\bar{v}-v_s)\tanh\left(\dfrac{x-v_s}{\bar{v}-v_s}\right)}{x},
& x>v_s,
\end{cases}
$}
\end{equation}
where $v_s$ is the speed threshold at which the scaling starts and satisfies
$\sup_w v_d(w)<v_s<\bar{v}$. The resulting vector field is $\overline{\boldsymbol{\chi}}^\mathrm{hgh}=\mu(\|\widetilde{\boldsymbol{\chi}}_p\|)\widetilde{\boldsymbol{\chi}}^\mathrm{hgh}$. This scaling only changes the magnitude of the vector field in \eqref{eq:proposed_vector_field}, without changing its direction or integral curves. On the other hand, the scaling guarantees that
$\|\overline{\boldsymbol{\chi}}_p\|<\bar{v}$,
and hence avoids generating physically infeasible reference velocities, where
$\overline{\boldsymbol{\chi}}_p= \mu(\|\widetilde{\boldsymbol{\chi}}_p\|)\widetilde{\boldsymbol{\chi}}_p\in \mathbb{R}^n$. We are now ready to present the control algorithm, where the acceleration input $\boldsymbol{a}$ is designed to satisfy the following properties.

\begin{theorem}\label{thm:002}
Consider the robot model in \eqref{eq:Second-Order}, and let a n-D desired path
$\mathcal{P}^{\mathrm{phy}}\subseteq\mathbb{R}^n$ with a path parametrization $\boldsymbol{f}: \mathbb{R}\to \mathbb{R}^n$
satisfy Assumption~\ref{ass:001}. Given a PPS, let the acceleration control law be $\boldsymbol{a}=\mathrm{Sat}_{\bar{a}}\left(\boldsymbol{\dot{\overline{\chi}}}_p-k_v\boldsymbol{e}_v\right)$, where $\boldsymbol{e}_v=\boldsymbol{v}-\overline{\boldsymbol{\chi}}_p$, $\boldsymbol{\dot{\overline{\chi}}}_p=\frac{\partial\overline{\boldsymbol{\chi}}_p}{\partial\boldsymbol{p}}\boldsymbol{v}+\frac{\partial\overline{\boldsymbol{\chi}}_p}{\partial w}\dot{w}$ and $k_v>0$ is a constant.
If $\boldsymbol{e}_v^\top\boldsymbol{\dot{\overline{\chi}}}_p\geq0$ holds during all acceleration-saturation periods, then the velocity-tracking error converges exponentially to zero, i.e., $\|\boldsymbol{e}_v(t)\|\to0$.
Moreover, as $t\to\infty$, the trajectory $\boldsymbol{p}(t)$ asymptotically converges to the desired path $\mathcal{P}^{\mathrm{phy}}$.
\end{theorem}

\begin{proof}
Define $\boldsymbol{y}= \boldsymbol{\dot{\overline{\chi}}}_p -k_v\boldsymbol{e}_v$ and $\alpha(t) = \min\left\{ 1,\frac{\bar a}{\Vert{}\boldsymbol{y}\Vert{}} \right\}$, where $\alpha=1$ when $\boldsymbol{y}=\boldsymbol{0}$. Then we have $\boldsymbol{a} = \alpha \left( \boldsymbol{\dot{\overline{\chi}}}_p -k_v\boldsymbol{e}_v \right),0<\alpha\leq1$. Consider the positive definite Lyapunov candidate
$V_1=\frac{1}{2}\boldsymbol{e}_v^\top \boldsymbol{e}_v$. Since $\dot{\boldsymbol{e}}_v =\boldsymbol{a} -\boldsymbol{\dot{\overline{\chi}}}_p$, we obtain $\dot V_1 = -\alpha k_v\Vert{}\boldsymbol{e}_v\Vert{}^2 -(1-\alpha) \boldsymbol{e}_v^\top \boldsymbol{\dot{\overline{\chi}}}_p$. In the unsaturated case, $\alpha=1$. During a saturation period, $\boldsymbol{e}_v^\top\boldsymbol{\dot{\overline{\chi}}}_p\geq0$
holds by assumption. Therefore, $\dot V_1 \leq -\alpha k_v\Vert{}\boldsymbol{e}_v\Vert{}^2 \leq0$, which implies  $\Vert{}\boldsymbol{e}_v(t)\Vert{} \leq\Vert{}\boldsymbol{e}_v(0)\Vert{}$. Moreover, by construction,
$\|\overline{\boldsymbol{\chi}}_p\|<\bar v$, and hence $\Vert{}\boldsymbol{v}\Vert{} \leq\bar v+\Vert{}\boldsymbol{e}_v(0)\Vert{}$.
Under Assumption~\ref{ass:001} and Definition~\ref{def:001},
$\boldsymbol{f}'$, $\boldsymbol{f}''$, $v_d$, and $v_d'$ are all bounded. In addition, the Jacobian of the radial scaling map
$\boldsymbol{z}\mapsto\mu(\|\boldsymbol{z}\|)\boldsymbol{z}$
is bounded, and $\dot w$ is also bounded. Therefore, there exists a constant $L_\chi>0$ such that $\Vert{}\boldsymbol{\dot{\overline{\chi}}}_p(t)\Vert{} \leq L_\chi$. It follows that $\Vert{}\boldsymbol{y}\Vert{} \leq L_\chi+k_v\Vert{}\boldsymbol{e}_v(0)\Vert{}$ and thus $\alpha(t) \geq \alpha_0= \min\left\{1,\frac{\bar a} {L_\chi+k_v\Vert{}\boldsymbol{e}_v(0)\Vert{}}\right\}>0$. Substituting this lower bound into the preceding expression gives $\dot V_1 \leq -2\alpha_0k_vV_1$. The above boundedness analysis also excludes finite-time escape of the system states, and hence guarantees that the closed-loop solution exists for all $t\geq0$. By the comparison lemma \cite[Lemma~3.4]{khalil2002nonlinear}, we have $\Vert{}\boldsymbol{e}_v(t)\Vert{} \leq \Vert{}\boldsymbol{e}_v(0)\Vert{} e^{-\alpha_0k_vt}$.
Therefore, $\|\boldsymbol{e}_v(t)\|$ converges exponentially to zero. 

We next consider the same Lyapunov candidate $V_2(\boldsymbol{\phi}) = \frac{1}{2}\boldsymbol{\phi}^\top \boldsymbol{K} \boldsymbol{\phi}$ as in Theorem~\ref{thm:main_convergence}. The time derivative of $\boldsymbol{\phi}$ is $\dot{\boldsymbol{\phi}} = \dot{\boldsymbol{p}} - \boldsymbol{f}'(w)\dot{w} = \mu \left( \widetilde{\boldsymbol{\chi}}_p - \boldsymbol{f}'(w) \widetilde{\chi}_w \right) + \boldsymbol{e}_v$. Therefore,
\begin{align}
\scalebox{1.0}{$\displaystyle \dot{V}_2$}
&=
\scalebox{1.0}{$\displaystyle
\boldsymbol{\phi}^\top \boldsymbol{K} \dot{\boldsymbol{\phi}}
=
\boldsymbol{\phi}^\top \boldsymbol{K}
\left(
\mu(\|\widetilde{\boldsymbol{\chi}}_p\|)
\left(
\widetilde{\boldsymbol{\chi}}_p
-
\boldsymbol{f}'(w)\widetilde{\chi}_w
\right)
+
\boldsymbol{e}_v
\right)
$}
\notag\\
&\overset{\eqref{eq:error}}{=}
\scalebox{1.0}{$\displaystyle
-\mu(\|\widetilde{\boldsymbol{\chi}}_p\|)
\boldsymbol{\phi}^\top \boldsymbol{K}
\left(
\boldsymbol{I}_n
+
\frac{
\boldsymbol{f}'(w)\boldsymbol{f}'(w)^\top
}{
\|\boldsymbol{f}'(w)\|^2
}
\right)
\boldsymbol{K}\boldsymbol{\phi}
+
\boldsymbol{\phi}^\top \boldsymbol{K}\boldsymbol{e}_v
$}
\notag\\
&\leq
\scalebox{1.0}{$\displaystyle
-\mu(\|\widetilde{\boldsymbol{\chi}}_p\|)
\|\boldsymbol{K}\boldsymbol{\phi}\|^2
+
\|\boldsymbol{K}\boldsymbol{\phi}\|
\|\boldsymbol{e}_v\|
$}.
\label{eq:v2dot}
\end{align}
From the first $n$ components of \eqref{eq:components_concise}, we have $\Vert{}\widetilde{\boldsymbol{\chi}}_p\Vert{} \leq \Vert{}\boldsymbol{K}\boldsymbol{\phi}\Vert{}+v_+$. By the definition of the scaling function in \eqref{eq:suofang}, $\mu(\|\widetilde{\boldsymbol{\chi}}_p\|)=1$ when $\|\widetilde{\boldsymbol{\chi}}_p\|\leq v_s$,
whereas $\mu(\Vert{}\widetilde{\boldsymbol{\chi}}_p\Vert{})\Vert{}\widetilde{\boldsymbol{\chi}}_p\Vert{}\geq v_s$
when $\|\widetilde{\boldsymbol{\chi}}_p\|>v_s$. Therefore, if $\Vert{}\boldsymbol{K}\boldsymbol{\phi}\Vert{}\geq R= \max\left\{v_+,\frac{v_s}{2}\right\}$, then $\mu(\Vert{}\widetilde{\boldsymbol{\chi}}_p\Vert{})\Vert{}\boldsymbol{K}\boldsymbol{\phi}\Vert{}\geq\frac{v_s}{2}$. Since $\boldsymbol{e}_v(t)\to0$, there exists $T>0$ such that $\Vert{}\boldsymbol{e}_v(t)\Vert{}\leq v_s/4$ for $t \ge T$. Hence, whenever $t\geq T$ and $\|\boldsymbol{K}\boldsymbol{\phi}\|\geq R$, it follows that $\dot V_2 \leq -\Vert{}\boldsymbol{K}\boldsymbol{\phi}\Vert{}\left( \mu(\Vert{}\widetilde{\boldsymbol{\chi}}_p\Vert{}) \Vert{}\boldsymbol{K}\boldsymbol{\phi}\Vert{}-\Vert{}\boldsymbol{e}_v\Vert{} \right) \leq -\frac{v_s}{4}\Vert{}\boldsymbol{K}\boldsymbol{\phi}\Vert{}<0$. Therefore, $\boldsymbol{\phi}(t)$ is bounded, and hence $\|\widetilde{\boldsymbol{\chi}}_p(t)\|$ is also bounded. Since $\mu(\|\widetilde{\boldsymbol{\chi}}_p\|)>0$ is continuous, there exists a constant $\mu_*>0$ such that $\mu(\Vert{}\widetilde{\boldsymbol{\chi}}_p(t)\Vert{})\geq\mu_*$. Using Young's inequality $ab\leq\frac{\varepsilon}{2}a^2 +\frac{1}{2\varepsilon}b^2,\varepsilon>0$
in \eqref{eq:v2dot} and choosing
$\varepsilon=\mu_*$, we obtain
\begin{align}
\scalebox{1.0}{$\displaystyle \dot V_2$}
&\scalebox{1.0}{$\displaystyle
\leq
-\frac{\mu_*}{2}\Vert{}\boldsymbol{K}\boldsymbol{\phi}\Vert{}^2
+\frac{1}{2\mu_*}\Vert{}\boldsymbol{e}_v\Vert{}^2
\overset{\eqref{eq:V_bounds}}{\leq}
-cV_2
+\frac{1}{2\mu_*}\Vert{}\boldsymbol{e}_v\Vert{}^2,
$}
\notag
\end{align}
where $c = \mu_*\lambda_{\min}^2/\lambda_{\max} >0$. Since $\|\boldsymbol{e}_v(t)\| \to 0$, the comparison lemma \cite[Lemma~3.4]{khalil2002nonlinear} yields $\lim_{t\to\infty}V_2(t)=0$
and therefore $\lim_{t\to\infty} \Vert{}\boldsymbol{\phi}(t)\Vert{}=0$.
Equivalently, as $t\to\infty$, the trajectory of the robot asymptotically converges to the physical desired path $\mathcal{P}^{\mathrm{phy}}$.
\end{proof}

\begin{remark}
The condition $\boldsymbol{e}_v^\top\boldsymbol{\dot{\overline{\chi}}}_p\geq0$
is only a sufficient condition for ensuring that the Lyapunov function continues to decrease during acceleration saturation, and is not a necessary condition for closed-loop stability. A temporary violation of this condition does not necessarily imply immediate instability of the closed-loop system. \hfill $\triangleleft$
\end{remark}


Theorem~\ref{thm:002} provides a sufficient condition for the asymptotic convergence of both the path error and the velocity-tracking error in the presence of acceleration saturation. However, it does not guarantee that the reference acceleration
$\|\boldsymbol{\dot{\overline{\chi}}}_p\|$
always satisfies the acceleration constraint $\bar a$. In particular, when the robot lies on the desired path and exactly tracks the vector field $\overline{\boldsymbol{\chi}}_p$, the reference acceleration consists of the tangential acceleration due to speed variation and the normal acceleration induced by the path curvature. An excessively large PPS may therefore cause acceleration saturation, especially on highly curved segments. To address this issue, we next design a curvature-aware PPS that reduces the speed in high-curvature regions while maintaining a high speed on low-curvature regions. To facilitate this design, we impose the following assumption.


\begin{assumption}\label{ass:003}
The function $\boldsymbol{f}(w)$ is $C^3$, and its curvature $\kappa(w)=\|\boldsymbol{f}^{\prime}(w)\times \boldsymbol{f}^{\prime\prime}(w)\|/\|\boldsymbol{f}^{\prime}(w)\|^3$ is continuously differentiable with respect to $w$. Moreover, there exists a constant $L_\kappa\geq0$ such that $|\kappa'(w)| \leq L_\kappa \|\boldsymbol{f}^{\prime}(w)\|$ for all $w \in \mathbb{R}$.
\end{assumption}

Assumption~\ref{ass:003} ensures that the path-curvature variation rate along the arc length is uniformly bounded, which enables the following curvature-aware PPS design.

\begin{theorem}[Curvature-aware PPS]\label{thm:003}
Consider the robot model in \eqref{eq:Second-Order} and the control law in Theorem~\ref{thm:002}. Let a 3D desired path $\mathcal{P}^{\mathrm{phy}} = \{\boldsymbol{f}(w) \in \mathbb{R}^3\mid w \in\mathbb{R}\}$ satisfy Assumptions~\ref{ass:001} and \ref{ass:003}. Choose a desired cruising speed $v_c\in(0,v_s)$ such that
$\bar{a}^2 \geq \frac{2 L_\kappa v_c^4}{3\sqrt{3}}$,
define $C=\frac{L_\kappa v_c^4}{3\sqrt{3}}$ and
$a_n = \sqrt{ \frac{\bar{a}^2 + \sqrt{\bar{a}^4 - 4 C^2}}{2}}$, and set
\begin{equation}\label{eq:adaptive_speed}
v_d(w)
=
\left(
\frac{1}{v_c^4}
+
\frac{\kappa^2(w)}{a_n^2}
\right)^{-\frac{1}{4}}.
\end{equation}
Then, when $\|\boldsymbol{\phi}(t)\|=0$ and $\|\boldsymbol{e}_v(t)\|=0$, it holds that $\|\boldsymbol{\dot{\overline{\chi}}}_p(t)\|\leq\bar{a}$. In this case, $\mathrm{Sat}_{\bar{a}}\left(\boldsymbol{\dot{\overline{\chi}}}_p-k_v\boldsymbol{e}_v\right) = \boldsymbol{\dot{\overline{\chi}}}_p$ and the reference acceleration is saturation-free.
\end{theorem}

\begin{proof}
Under Assumptions \ref{ass:001} and \ref{ass:003}, the function in \eqref{eq:adaptive_speed} satisfies Definition \ref{def:001} and hence is a valid PPS. When $\|\boldsymbol{\phi}(t)\|=0$ and $\|\boldsymbol{e}_v(t)\|=0$, the robot lies exactly on the desired path and perfectly tracks the vector field
$\overline{\boldsymbol{\chi}}_p$. The actual velocity of the robot is $\boldsymbol{v} = \overline{\boldsymbol{\chi}}_p = v_d(w) \frac{\boldsymbol{f}'(w)}{\|\boldsymbol{f}'(w)\|} = v_d(w) \boldsymbol{T}(w)$, where $\boldsymbol{T}(w)$ is the unit tangent vector. The reference acceleration is $\dot{\boldsymbol{v}}=\boldsymbol{\dot{\overline{\chi}}}_p  = \dot{v}_d \boldsymbol{T} + v_d \dot{\boldsymbol{T}}$. Since $\|\boldsymbol T\|=1$, it follows that
$\boldsymbol T^\top\dot{\boldsymbol T}=0$. Furthermore, let $s$ denote the arc length, such that $\left\Vert{}\frac{d\boldsymbol{T}}{ds}\right\Vert{}=\kappa$. Applying $\dot s=v_d$ yields $\|\dot{\boldsymbol T}\|=\kappa v_d$. Therefore,
\begin{equation}\label{eq:acc_squared_proof}
    \|\boldsymbol{\dot{\overline{\chi}}}_p\|^2 = \dot{v}_d^2 + (\kappa v_d^2)^2.
\end{equation}


\textit{1) Bounded normal acceleration:}
Substituting \eqref{eq:adaptive_speed} into \eqref{eq:acc_squared_proof} gives
\begin{equation}\label{eq:normal_bound}
\scalebox{1.0}{$\displaystyle
(\kappa v_d^2)^2
=
\frac{\kappa^2}{v_c^{-4}+\kappa^2/a_n^2}
=
a_n^2
\frac{\kappa^2/a_n^2}{v_c^{-4}+\kappa^2/a_n^2}
\leq
a_n^2
$}.
\end{equation}

\textit{2) Bounded tangential acceleration:}
Using the chain rule
$\dot v_d=v_d'(w)\dot w$
and
$|\kappa'(w)|
\leq
L_\kappa\|\boldsymbol{f}'(w)\|$
from Assumption~\ref{ass:003}, we obtain
\begin{align}
\scalebox{1.0}{$\displaystyle |\dot{v}_d|$}
&\scalebox{1.0}{$\displaystyle
= \frac{|\kappa\kappa'|}{2a_n^2\|\boldsymbol{f}'(w)\|}
\left(
v_c^{-4}+\frac{\kappa^2}{a_n^2}
\right)^{-\frac{3}{2}}
\leq
\frac{C}{a_n}.$}
\label{eq:tangential_bound}
\end{align}
Here, we have used the fact that the function $g(x) = x / (v_c^{-4} + x^2/a_n^2)^{3/2}$ attains its maximum $\frac{2 a_n v_c^4}{3\sqrt{3}}$ at
$x = \frac{a_n}{\sqrt{2}v_c^2}$.

\textit{3) Bounded total acceleration:}
Note that $a_n^2$ is a real root of $x^2 - \bar{a}^2 x + C^2 = 0$, whose existence is guaranteed by $\bar{a}^2\geq 2C$. Combining
\eqref{eq:normal_bound} and \eqref{eq:tangential_bound} with
\eqref{eq:acc_squared_proof} gives $\|\boldsymbol{\dot{\overline{\chi}}}_p\|^2 \leq \frac{C^2}{a_n^2} + a_n^2 =\bar{a}^2$. Therefore, $\Vert{}\boldsymbol{\dot{\overline{\chi}}}_p\Vert{} \leq \bar{a}$.
\end{proof}


\section{Simulation Comparison and Experiments}

\subsection{Simulation Comparison}
\begin{figure*}[!t]
    \centering
    \subfloat[]{
        \includegraphics[width=0.23\linewidth]{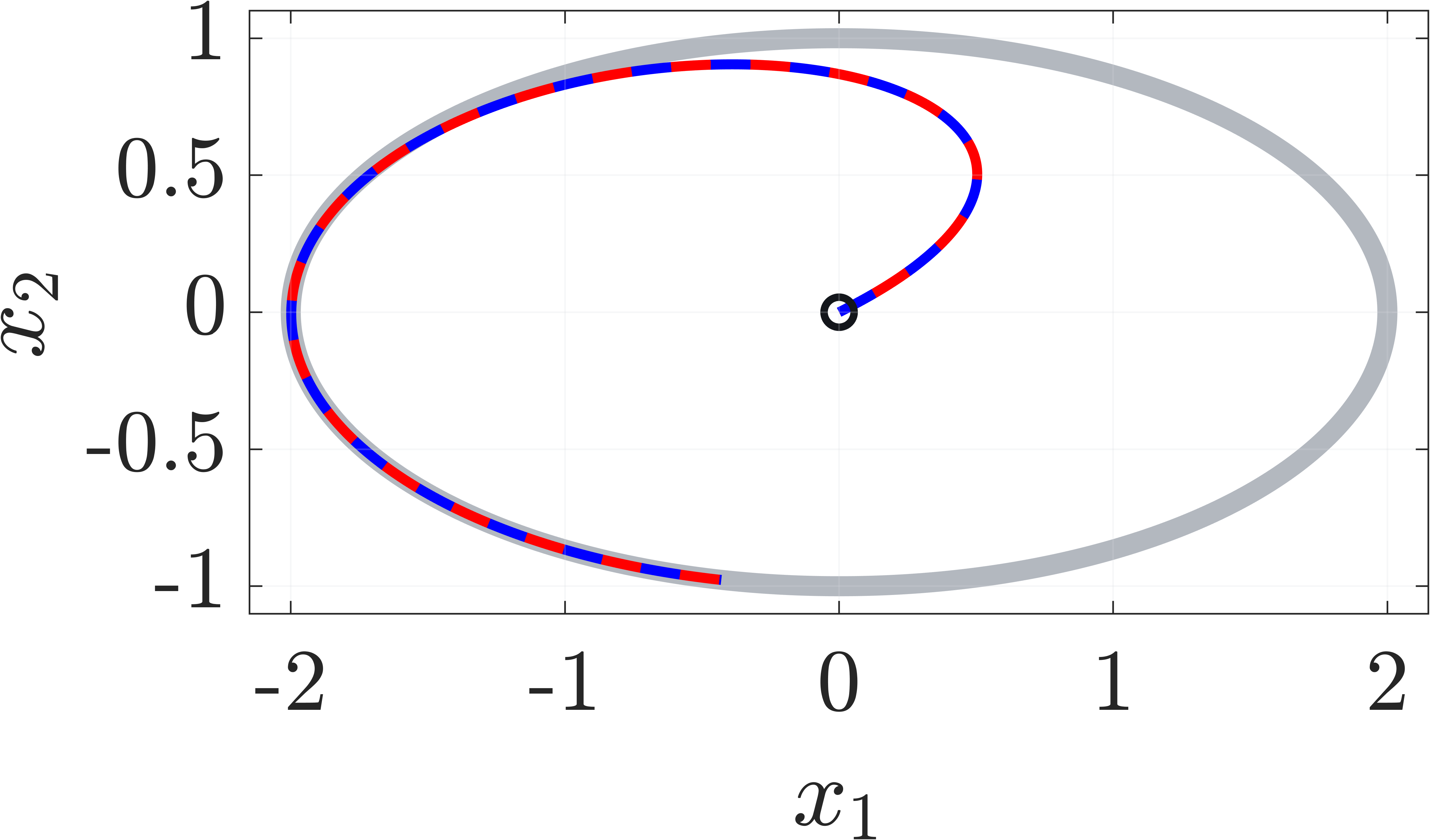}
        \label{fig:004a}
    }
    \subfloat[]{
        \includegraphics[width=0.23\linewidth]{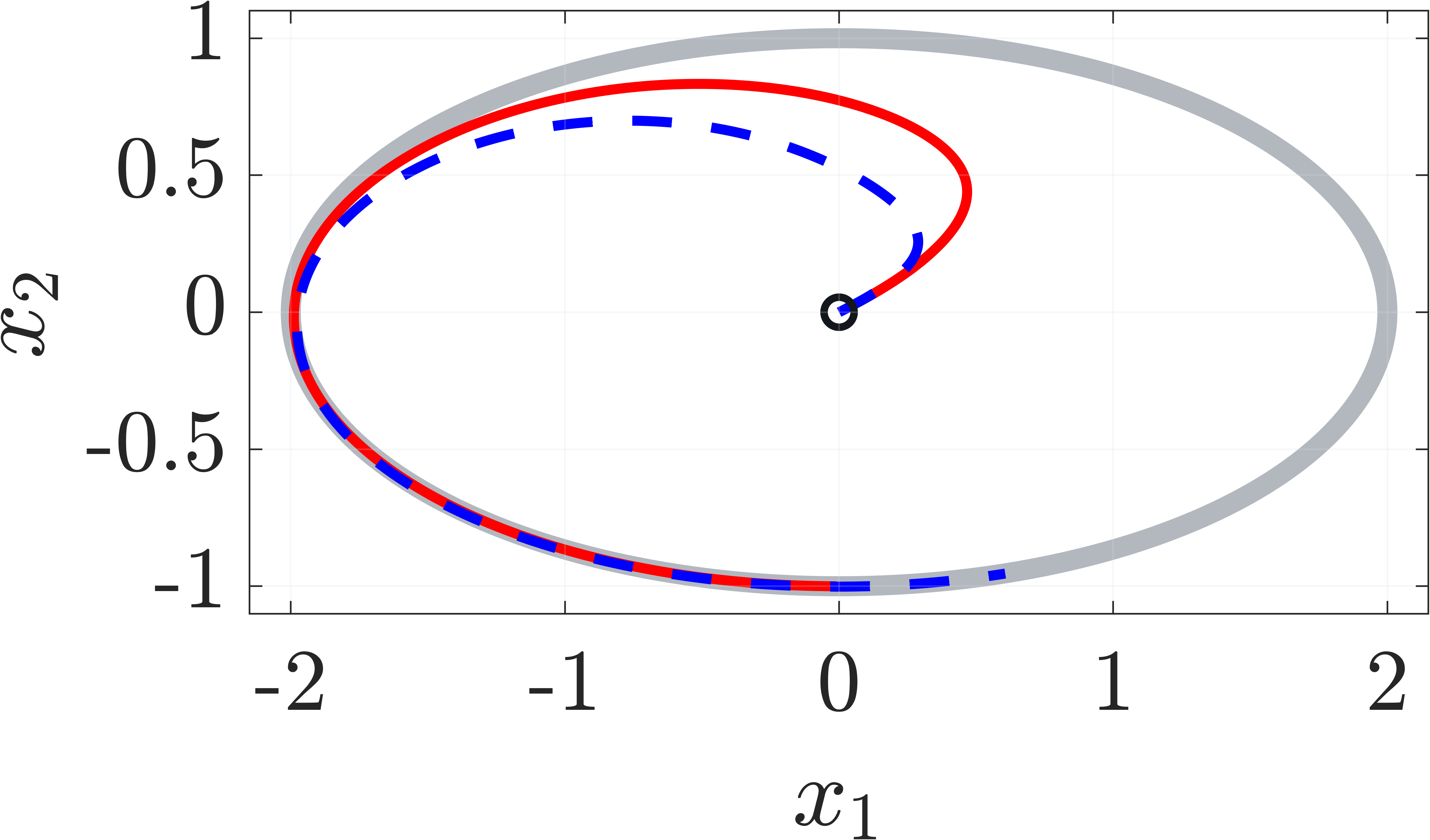}
        \label{fig:004b}
    }
    \subfloat[]{
        \includegraphics[width=0.23\linewidth]{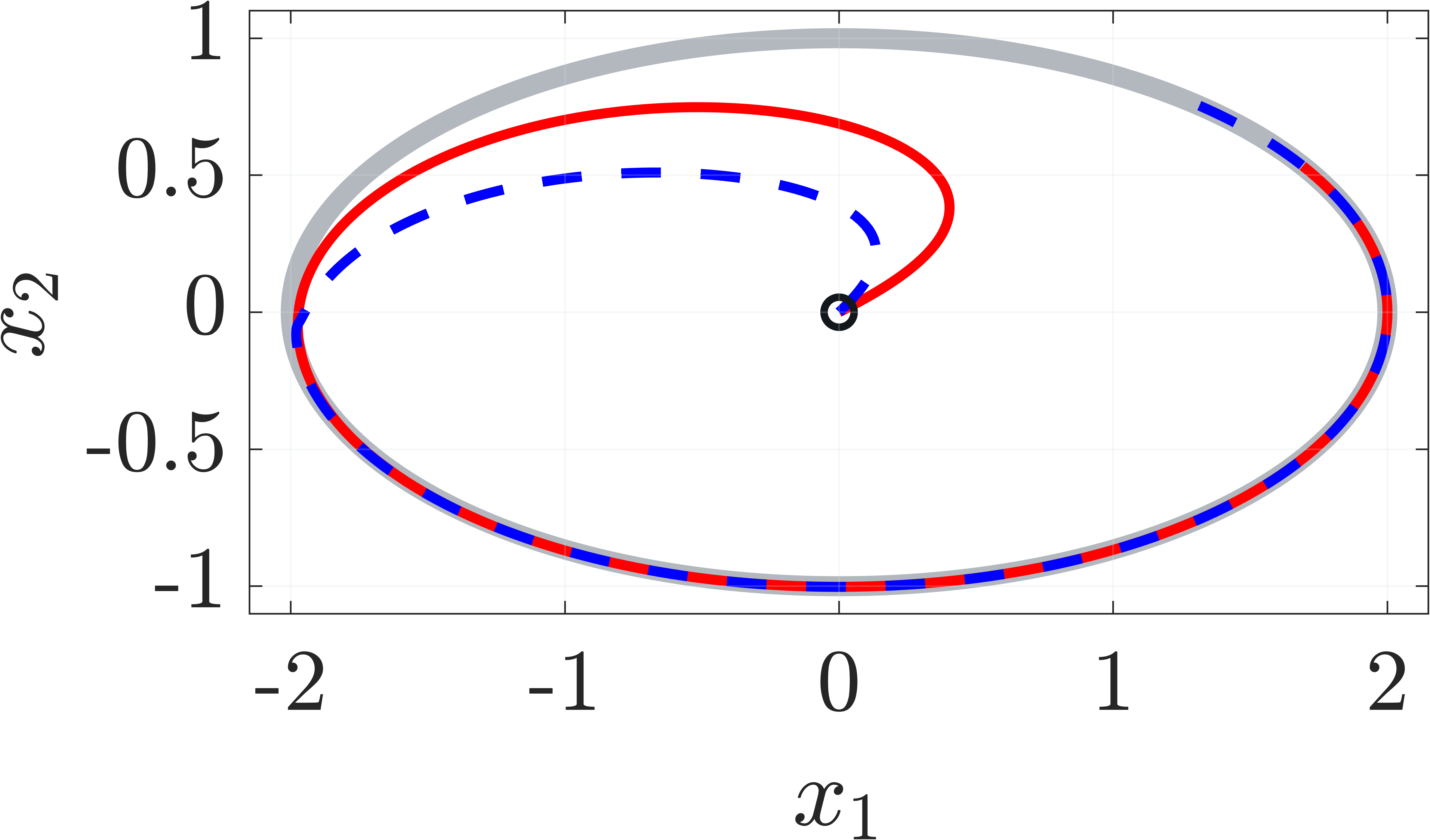}
        \label{fig:004c}
    }
    \subfloat[]{
        \includegraphics[width=0.23\linewidth]{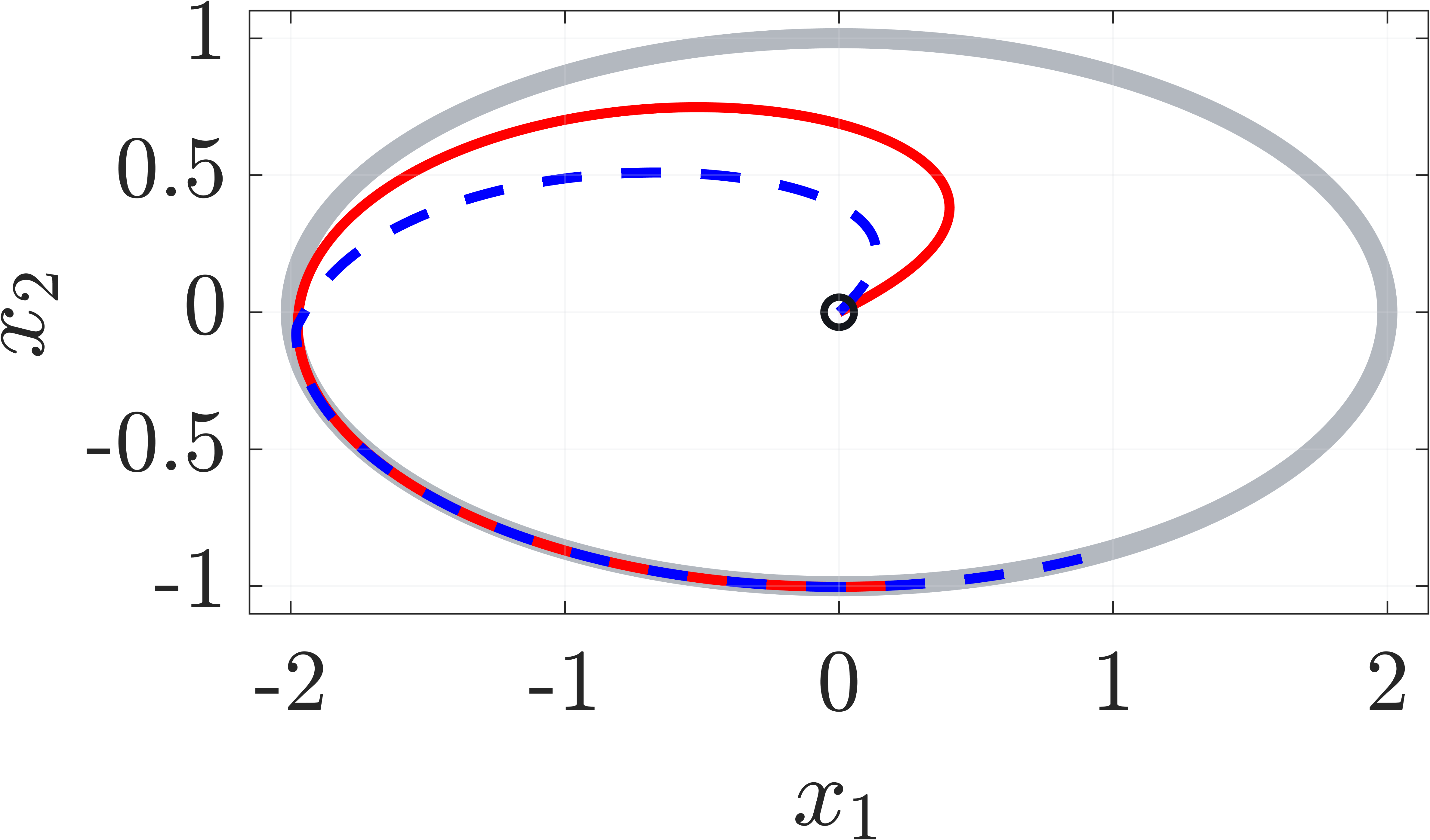}
        \label{fig:004d}
    }
    
    \subfloat[]{
        \includegraphics[width=0.3\linewidth]{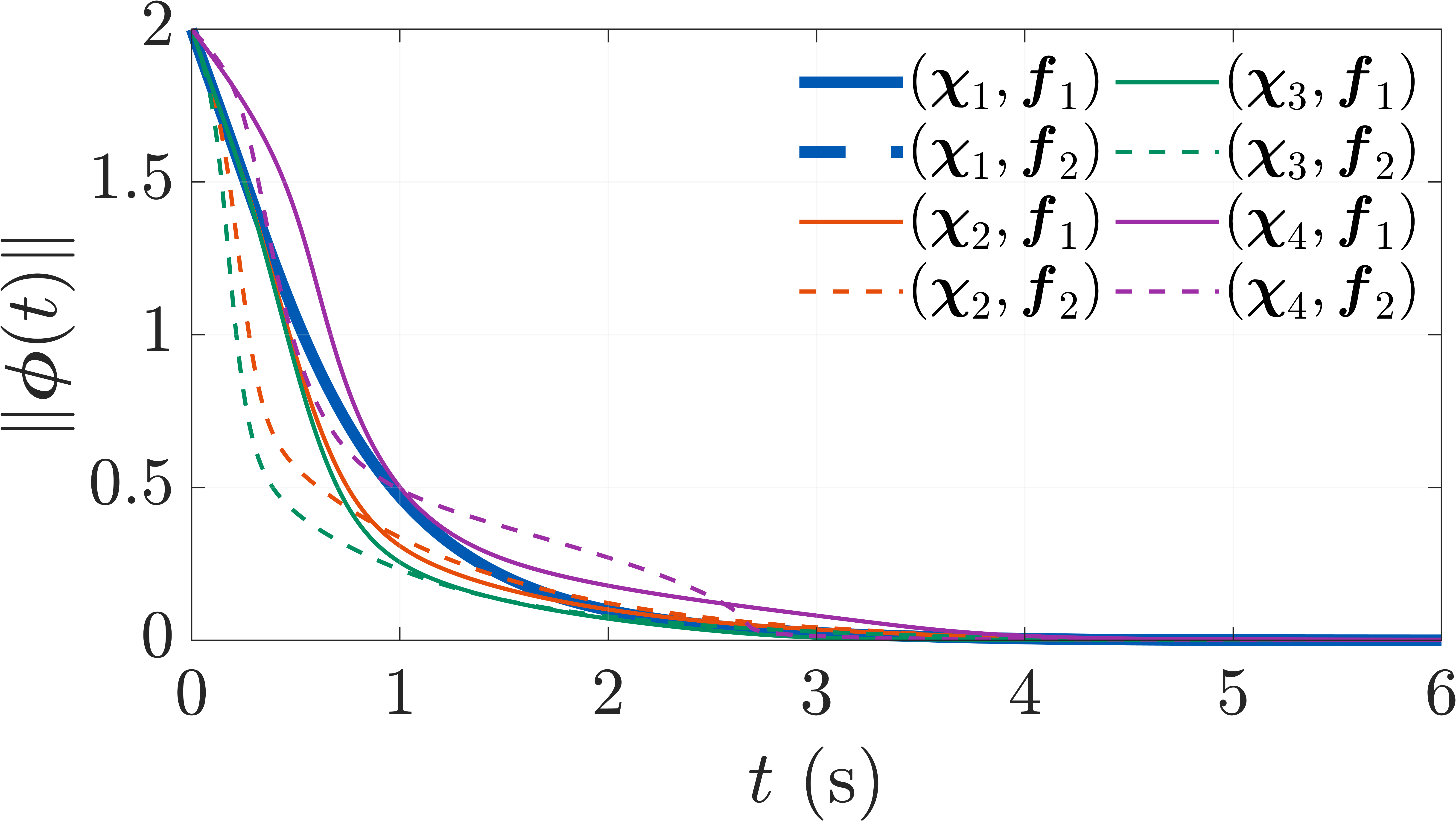}
        \label{fig:004e}
    }
    \subfloat[]{
        \includegraphics[width=0.3\linewidth]{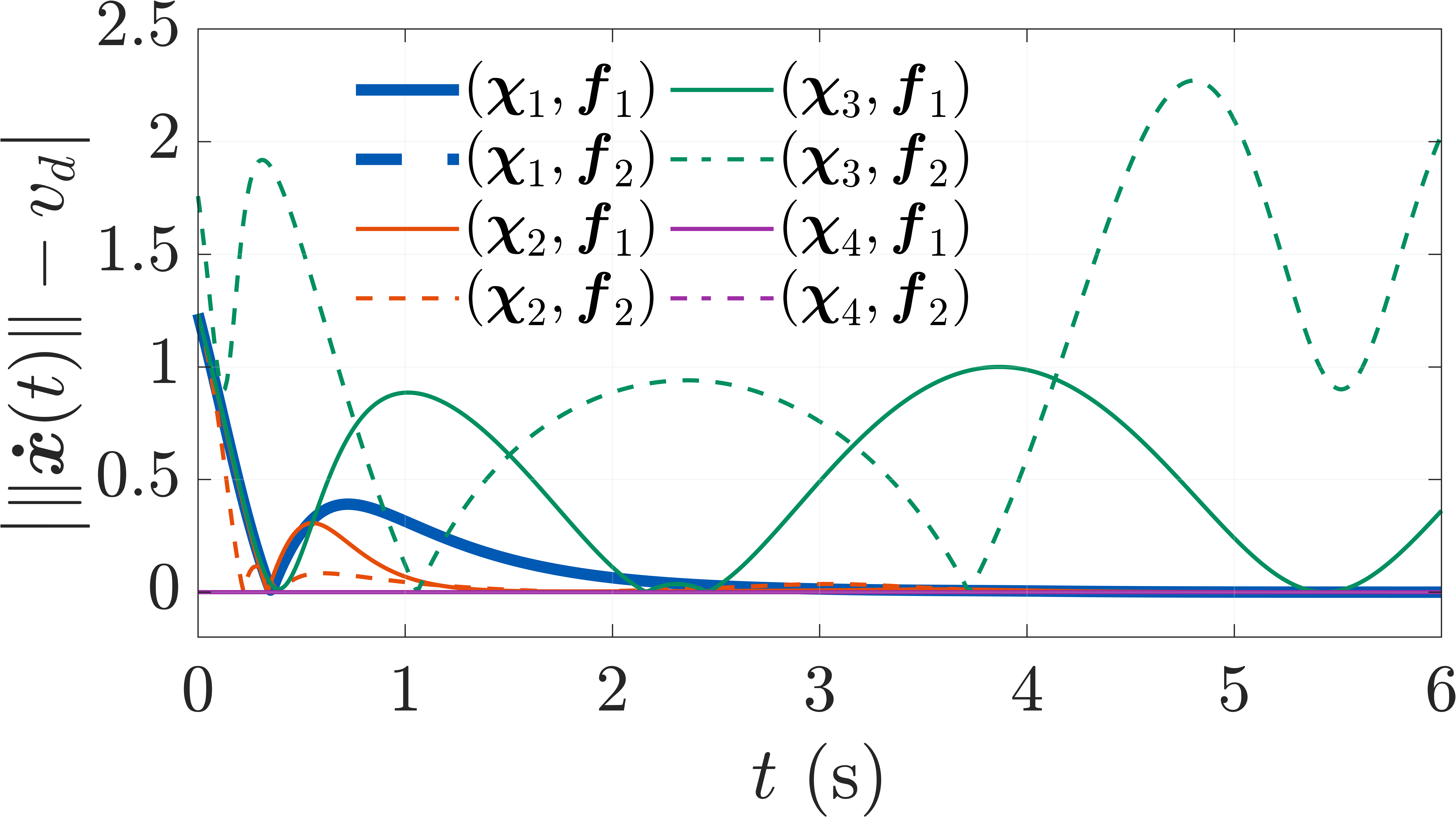}
        \label{fig:004f}
    }
    \subfloat[]{
        \includegraphics[width=0.3\linewidth]{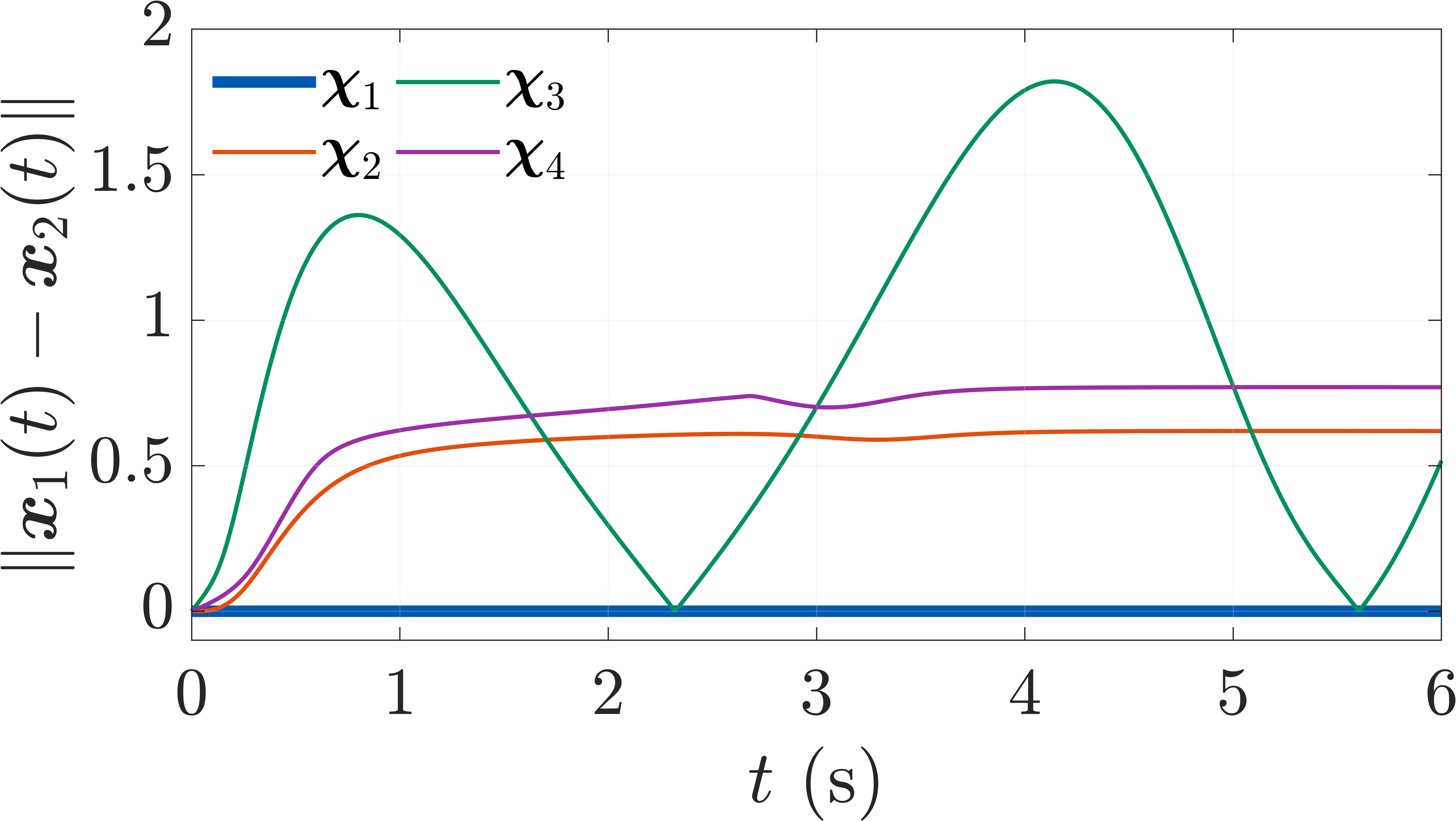}
        \label{fig:004g}
    }
    \caption{Results of the simulation comparison for two different regular parameterizations of the same ellipse. (a)–(d) correspond to $\boldsymbol{\chi}_1$–$\boldsymbol{\chi}_4$, respectively, where the thick gray curve denotes the desired ellipse, and the red solid and blue dashed curves denote the trajectories generated under the parameterizations $\boldsymbol{f}_1$ and $\boldsymbol{f}_2$, respectively. (e)–(g) show the path error $\|\boldsymbol{\phi}(t)\|$, the physical-speed error $\big|\|\dot{\boldsymbol{x}}(t)\|-v_d(w(t))\big|$, and the Euclidean distance $\|\boldsymbol{x}_1(t)-\boldsymbol{x}_2(t)\|$ between the physical trajectories generated by the two parameterizations, respectively.}
    \label{fig:004}
\end{figure*}

To verify the invariance of the path-error dynamics under different parameterizations stated in Corollary~\ref{cor:001} and to examine the role of the matrix $\boldsymbol{M}(w)$, we compare the proposed vector field with three benchmark vector fields. For convenience, the four vector fields are denoted by $\boldsymbol{\chi}_i$, $i=1,2,3,4$. Specifically, let $\boldsymbol{\chi}_1=\widetilde{\boldsymbol{\chi}}^{\mathrm{hgh}}$, where $\widetilde{\boldsymbol{\chi}}^{\mathrm{hgh}}$ is given by \eqref{eq:proposed_vector_field}. Let $\boldsymbol{\chi}_2$ be an ablated version of $\boldsymbol{\chi}_1$ obtained by replacing $\boldsymbol{M}(w)$ with the identity matrix $\boldsymbol{I}_{n+1}$, which is used to evaluate the role of $\boldsymbol{M}(w)$ in reparameterization invariance. Furthermore, let $\boldsymbol{\chi}_3=\boldsymbol{\chi}^{\mathrm{hgh}}$ and $\boldsymbol{\chi}_4=\widehat{\boldsymbol{\chi}}^{\mathrm{hgh}}$, corresponding to the SF-GVF in \eqref{eq:high_vector_field} and the partially normalized SF-GVF in \eqref{eq:partone}, respectively.

\begin{table}[!h]
    \centering
    \scriptsize
    \caption{Comparison of the four SF-GVFs.}
    \label{tab:comparison_gvf}
    \setlength{\tabcolsep}{2.0pt}
    \renewcommand{\arraystretch}{1.05}
    \begin{tabular}{c c c c c}
        \toprule
        \multirow{2}{*}{\textbf{SF-GVF}}
        & \textbf{global}
        & \textbf{physical speed}
        & \textbf{globally}
        & \textbf{reparam.-invariant}
        \\
        & \textbf{convergence}
        & \textbf{regulation}
        & \textbf{well-defined}
        & \textbf{path-error dynamics}
        \\
        \midrule
        $\boldsymbol{\chi}_1$
        & \cmark
        & \cmark
        & \cmark
        & \cmark
        \\

        $\boldsymbol{\chi}_2$
        & \cmark
        & \cmark
        & \cmark
        & \xmark
        \\

        $\boldsymbol{\chi}_3$
        & \cmark
        & \xmark
        & \cmark
        & \xmark
        \\

        $\boldsymbol{\chi}_4$
        & \xmark
        & \cmark
        & \xmark
        & \xmark
        \\
        \bottomrule
    \end{tabular}

\end{table}

Consider a 2D elliptical path $\mathcal{P}^{\mathrm{phy}}$ with the following two different parameterizations:
\begin{equation*}
\scalebox{1.0}{$\displaystyle
\begin{aligned}
\boldsymbol{f}_1(w)&=
\begin{bmatrix}
2\cos w \\ \sin w
\end{bmatrix}, 
&
\boldsymbol{f}_2(w)&=
\begin{bmatrix}
2\cos(w+0.9\sin w) \\ \sin(w+0.9\sin w)
\end{bmatrix}.
\end{aligned}
$}
\end{equation*}
Clearly, $\boldsymbol{f}_1$ and $\boldsymbol{f}_2$ describe the same geometric ellipse but with different parameterizations. The same parameters and initial condition are used for all four vector fields. Specifically, we choose
$\boldsymbol{K}=\boldsymbol{I}_2$,
$\boldsymbol{\xi}(0)=\begin{bmatrix}0,0,0\end{bmatrix}^{\top}$,
and $v_d(w)\equiv1$. The simulation time is $6~\mathrm{s}$.

The results are shown in Fig.~\ref{fig:004}. Under the two different parameterizations, the physical trajectories generated by $\boldsymbol{\chi}_1$ coincide exactly, and the physical speed eventually converges to the PPS. The properties of the four guiding vector fields are summarized in Table~\ref{tab:comparison_gvf}. In particular, the comparison between $\boldsymbol{\chi}_1$ and $\boldsymbol{\chi}_2$ shows that the matrix $\boldsymbol{M}(w)$ plays a key role in achieving reparameterization-invariant path-error dynamics. The results of $\boldsymbol{\chi}_3$ show that the conventional SF-GVF itself cannot independently regulate the physical speed. Although $\boldsymbol{\chi}_4$ can regulate the physical speed within its domain, its normalization denominator may vanish, rendering the vector field undefined, which is consistent with our previous analysis in Sections \ref{sec:002}--\ref{sec:003}.

\subsection{Experiment 1: Path Following with Constant PPS}\label{sec:exp1}

\begin{figure}[!t]
    \centering
    \subfloat[]{
        \includegraphics[width=0.45\linewidth]{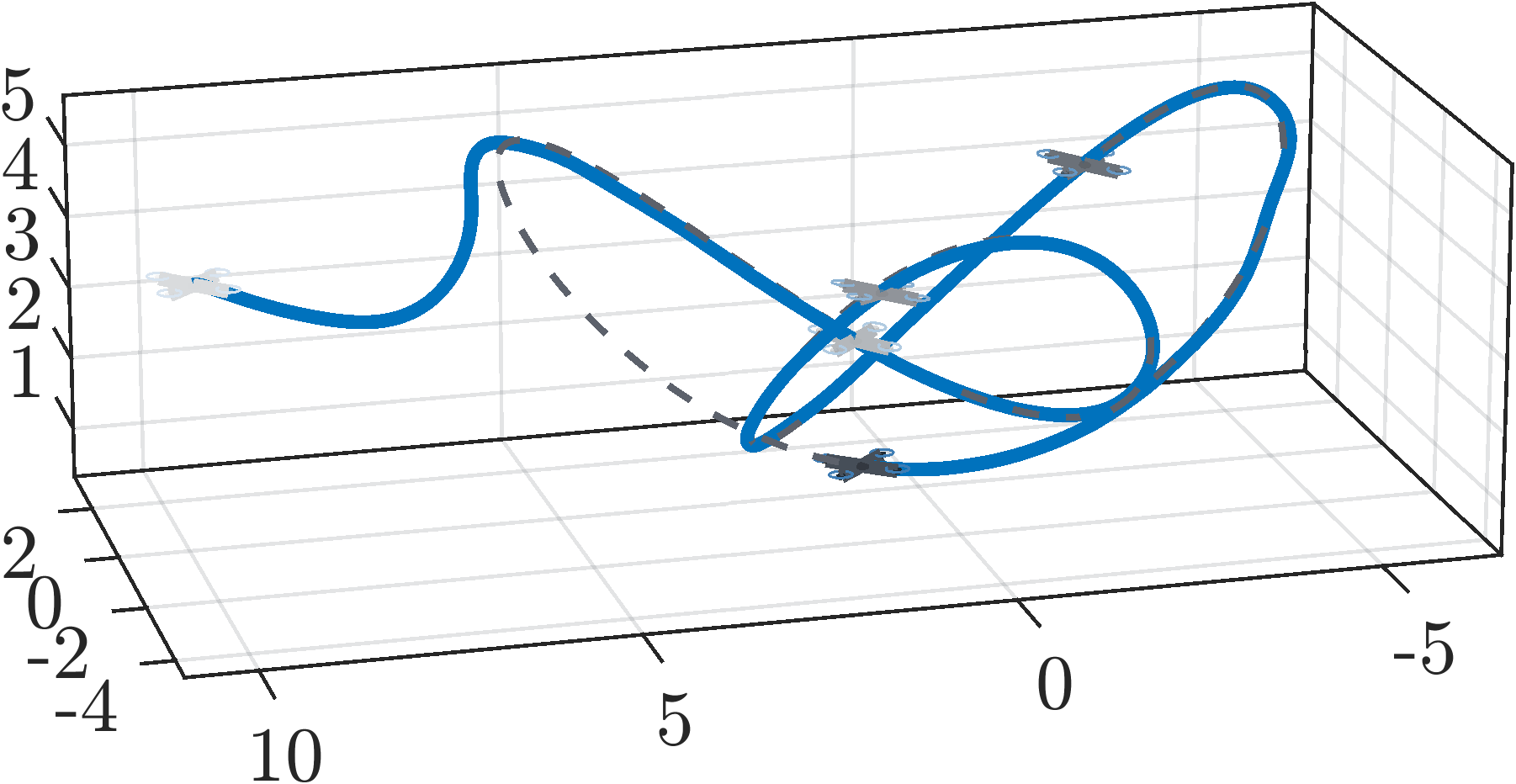}
        \label{fig:006a}
    }
    \subfloat[]{
        \includegraphics[width=0.45\linewidth]{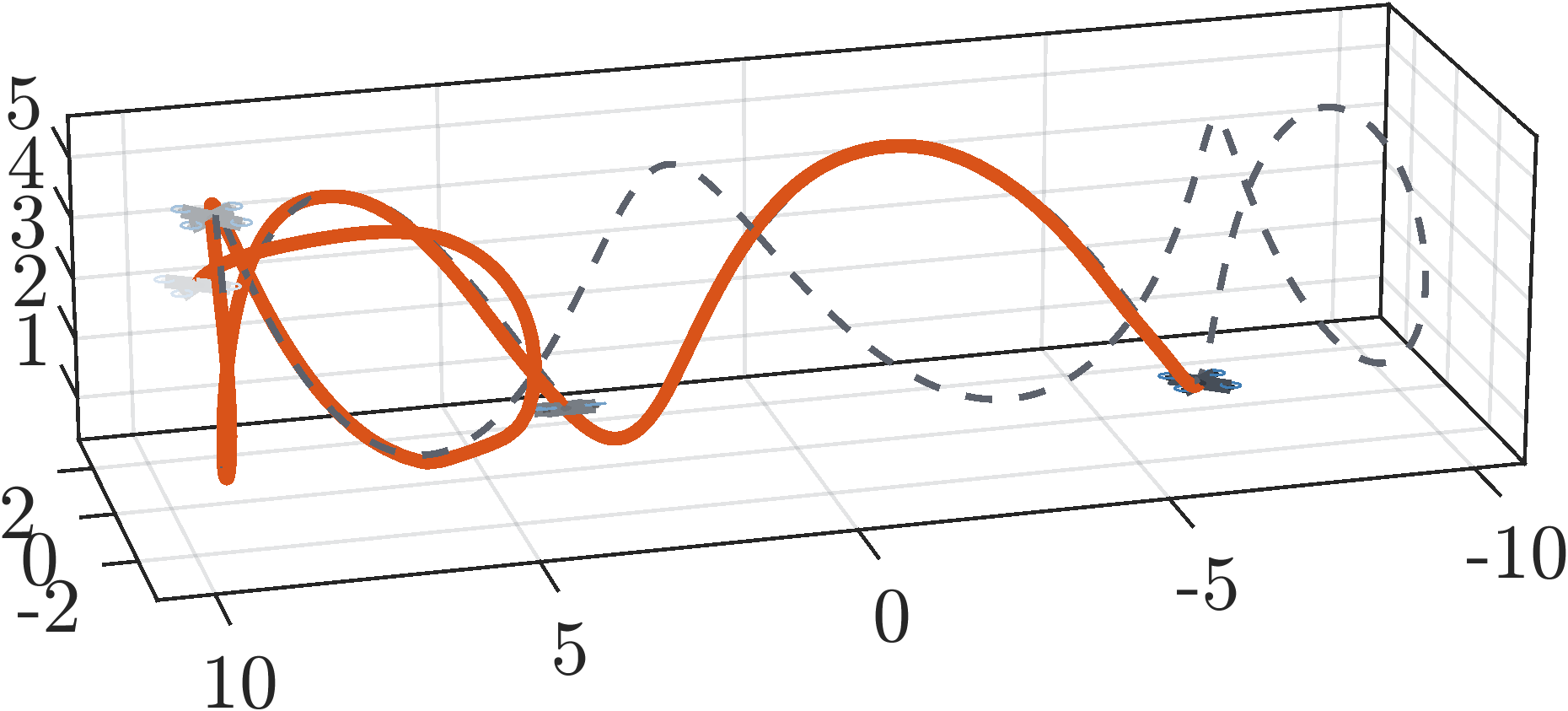}
        \label{fig:006b}
    }

    \subfloat[]{
        \includegraphics[width=0.45\linewidth]{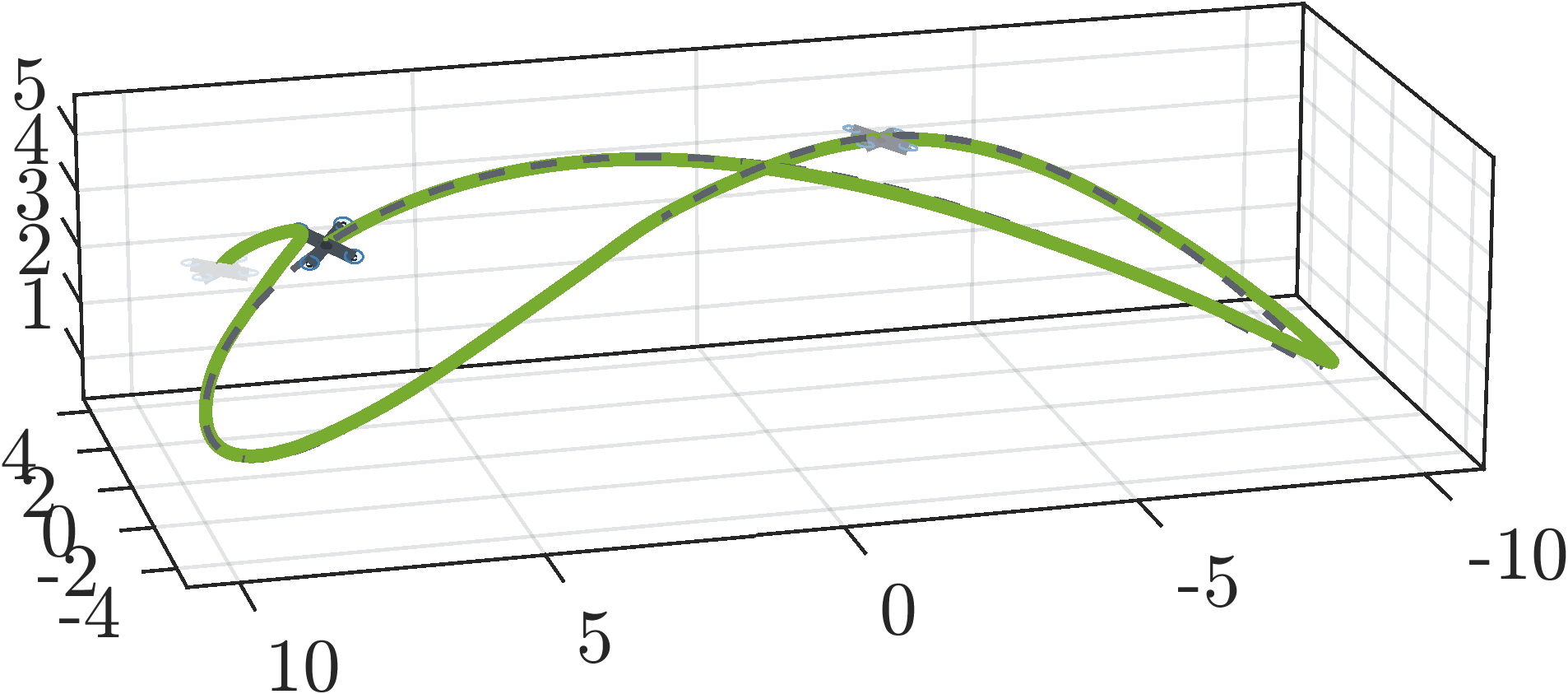}
        \label{fig:006c}
    }
    \subfloat[]{
        \includegraphics[width=0.45\linewidth]{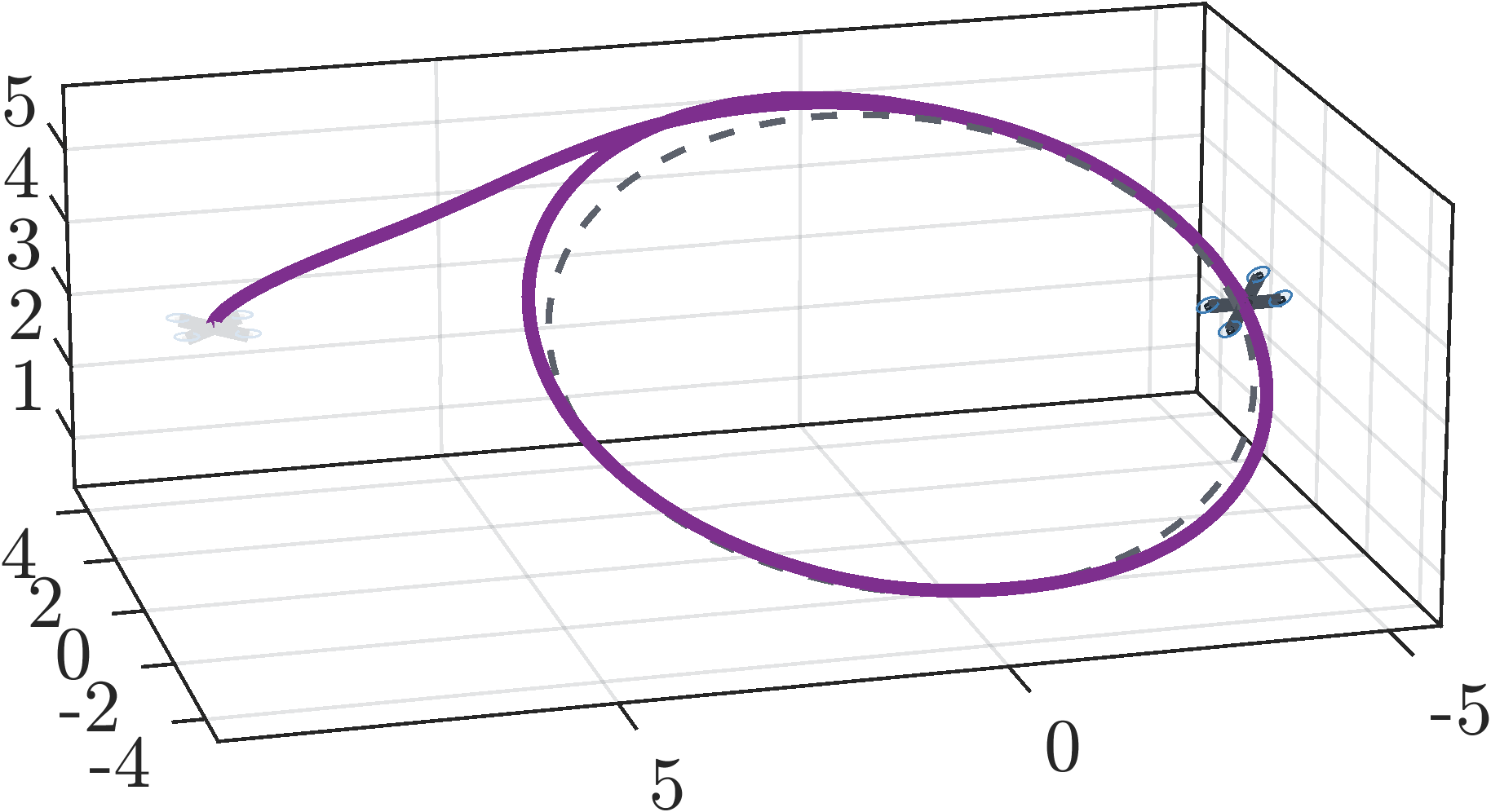}
        \label{fig:006d}
    }
    \caption{The flight trajectories of the quadrotor in Experiment 1. (a)–(d) correspond to the desired path parameterizations $\boldsymbol{f}_1$--$\boldsymbol{f}_4$, with flight times of $30\,\mathrm{s}$, $30\,\mathrm{s}$, $15\,\mathrm{s}$, and $10\,\mathrm{s}$, respectively. The black dashed curves denote the desired paths, and the solid curves denote the actual trajectories.}
    \label{fig:006}
\end{figure}

\begin{figure}[!t]
    \centering
    \subfloat[]{
        \includegraphics[width=0.45\linewidth]{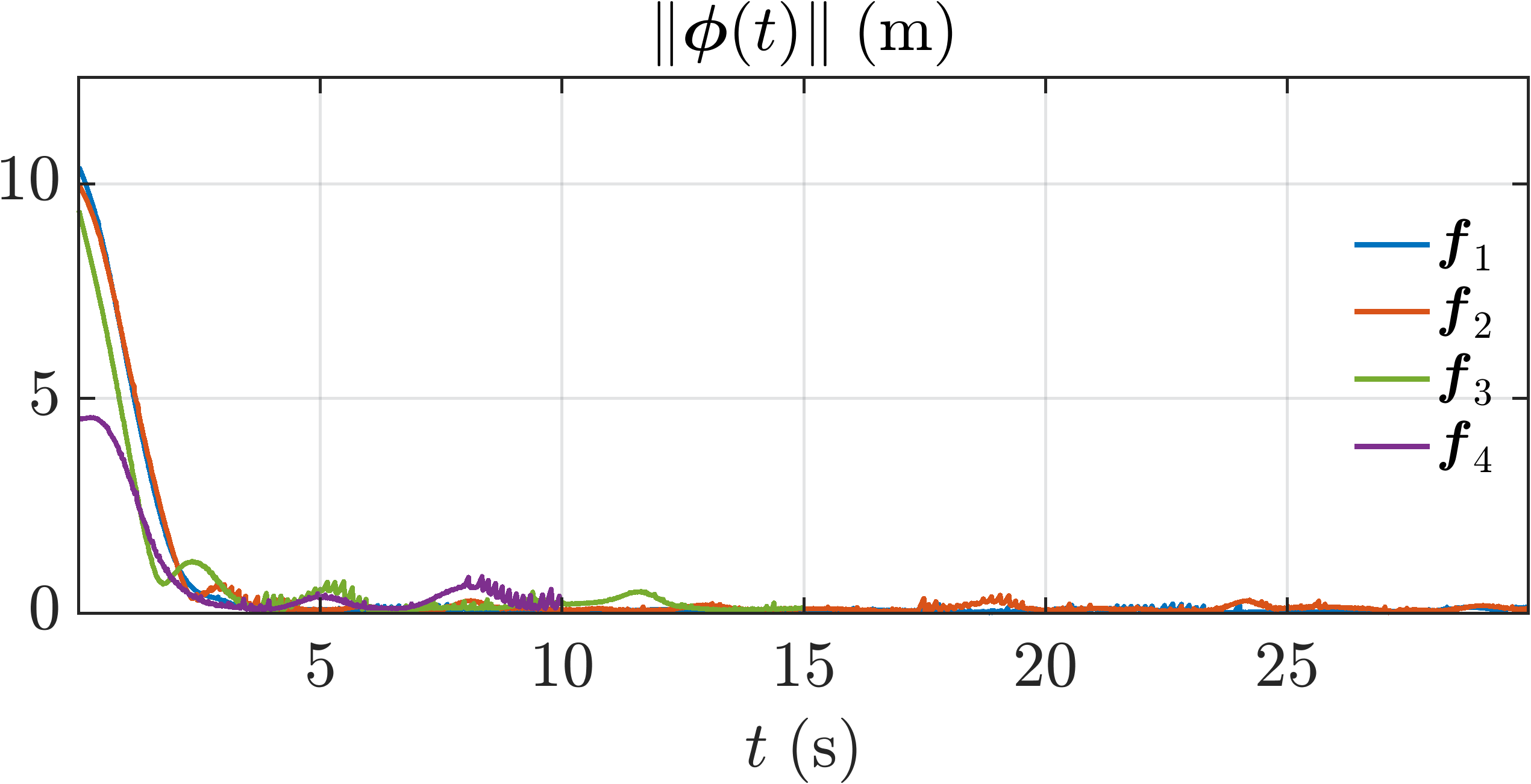}
        \label{fig:007a}
    }
    \subfloat[]{
        \includegraphics[width=0.45\linewidth]{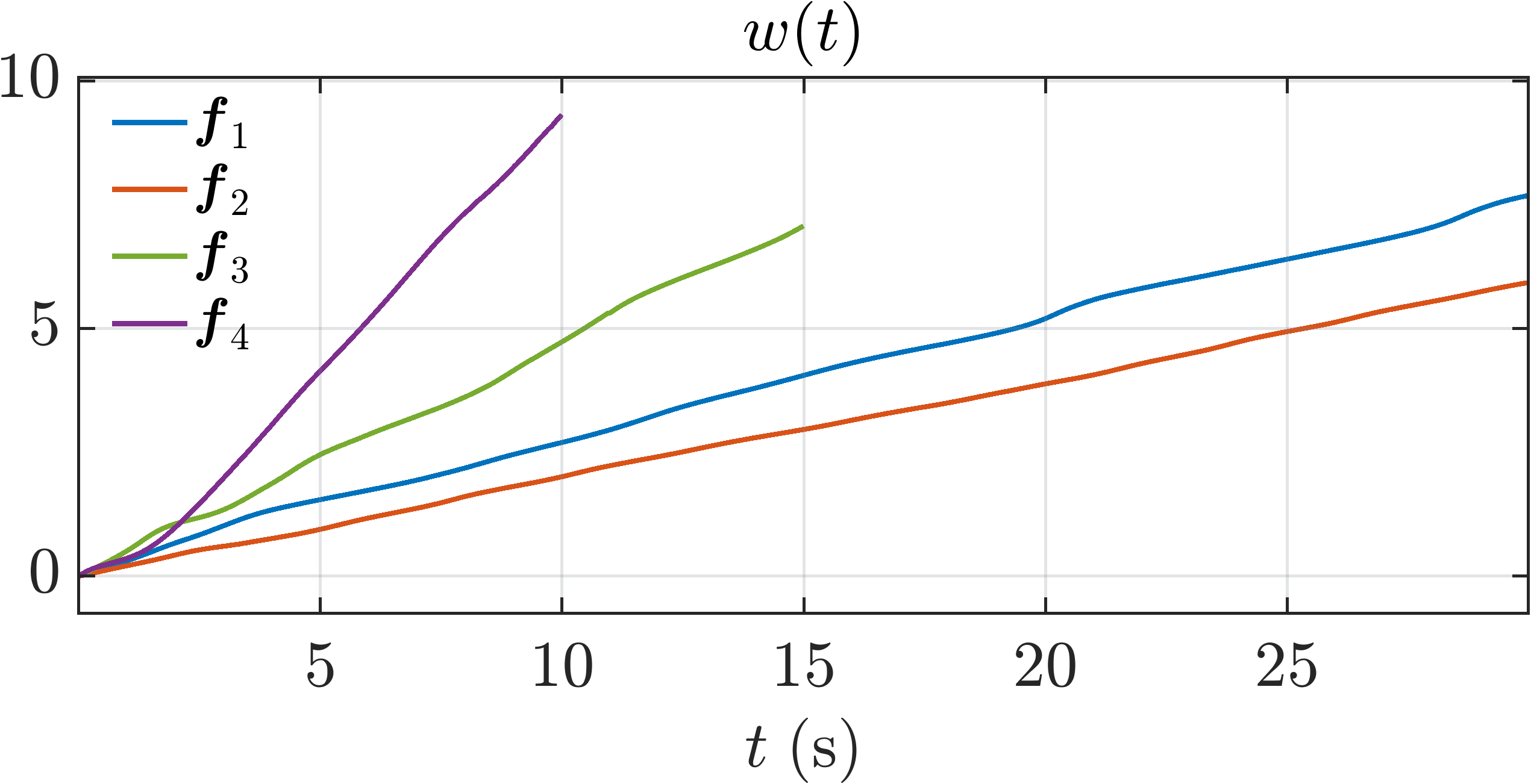}
        \label{fig:007b}
    }

    \subfloat[]{
        \includegraphics[width=0.45\linewidth]{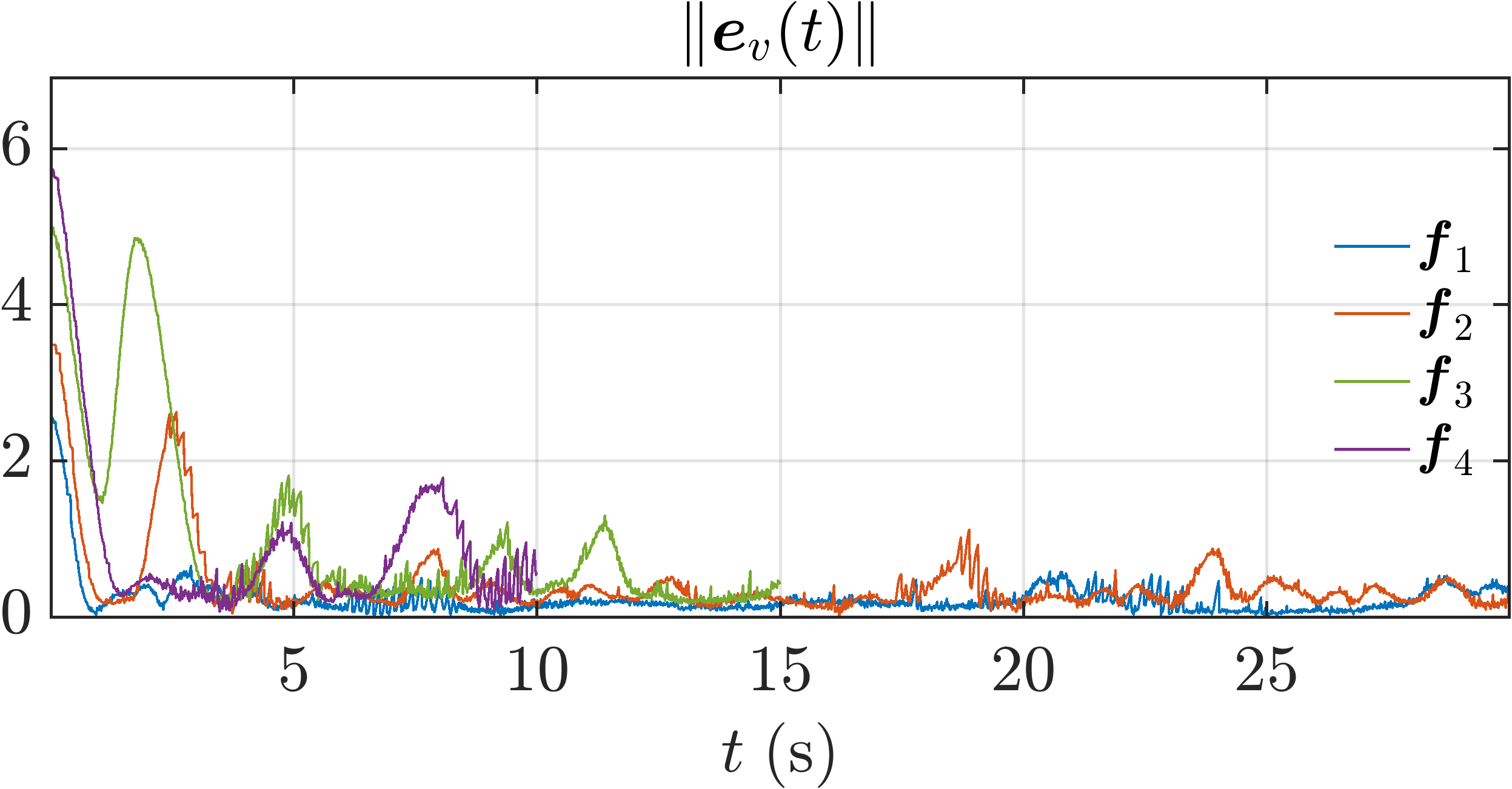}
        \label{fig:007c}
    }
    \subfloat[]{
        \includegraphics[width=0.45\linewidth]{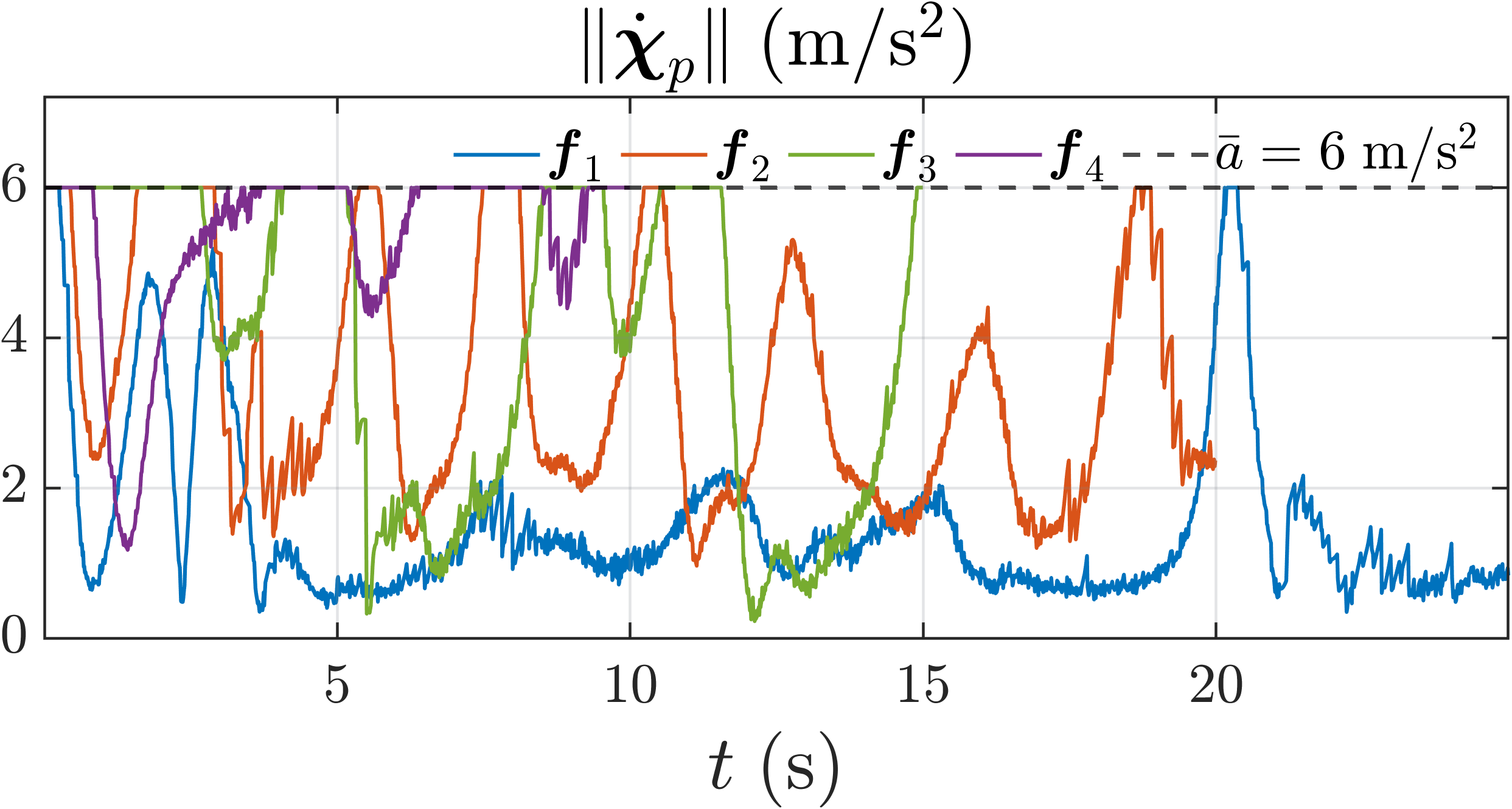}
        \label{fig:007d}
    }
    \caption{Results of the first experiment. (a)–(d) show the path error $\|\boldsymbol{\phi}(t)\|$, path parameter $w(t)$, velocity-tracking error $\|\boldsymbol{e}_v(t)\|$, and reference acceleration $\|\boldsymbol{\dot{\chi}}_p\|$, respectively, for the four desired paths.}
    \label{fig:007}
\end{figure}

\begin{figure*}[!t]
    \centering
    \subfloat[]{
        \includegraphics[width=0.23\linewidth]{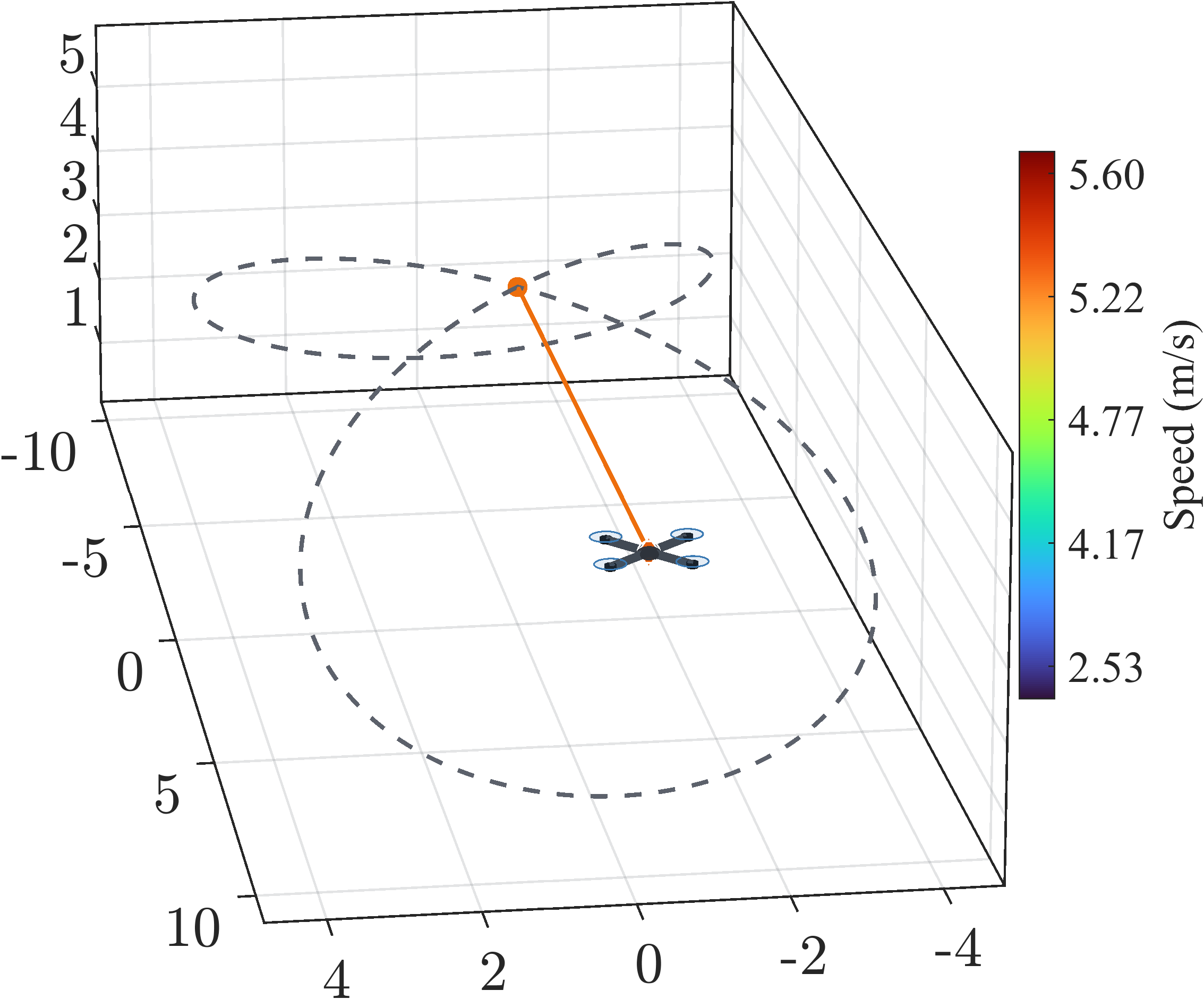}
        \label{fig:008a}
    }
    \subfloat[]{
        \includegraphics[width=0.23\linewidth]{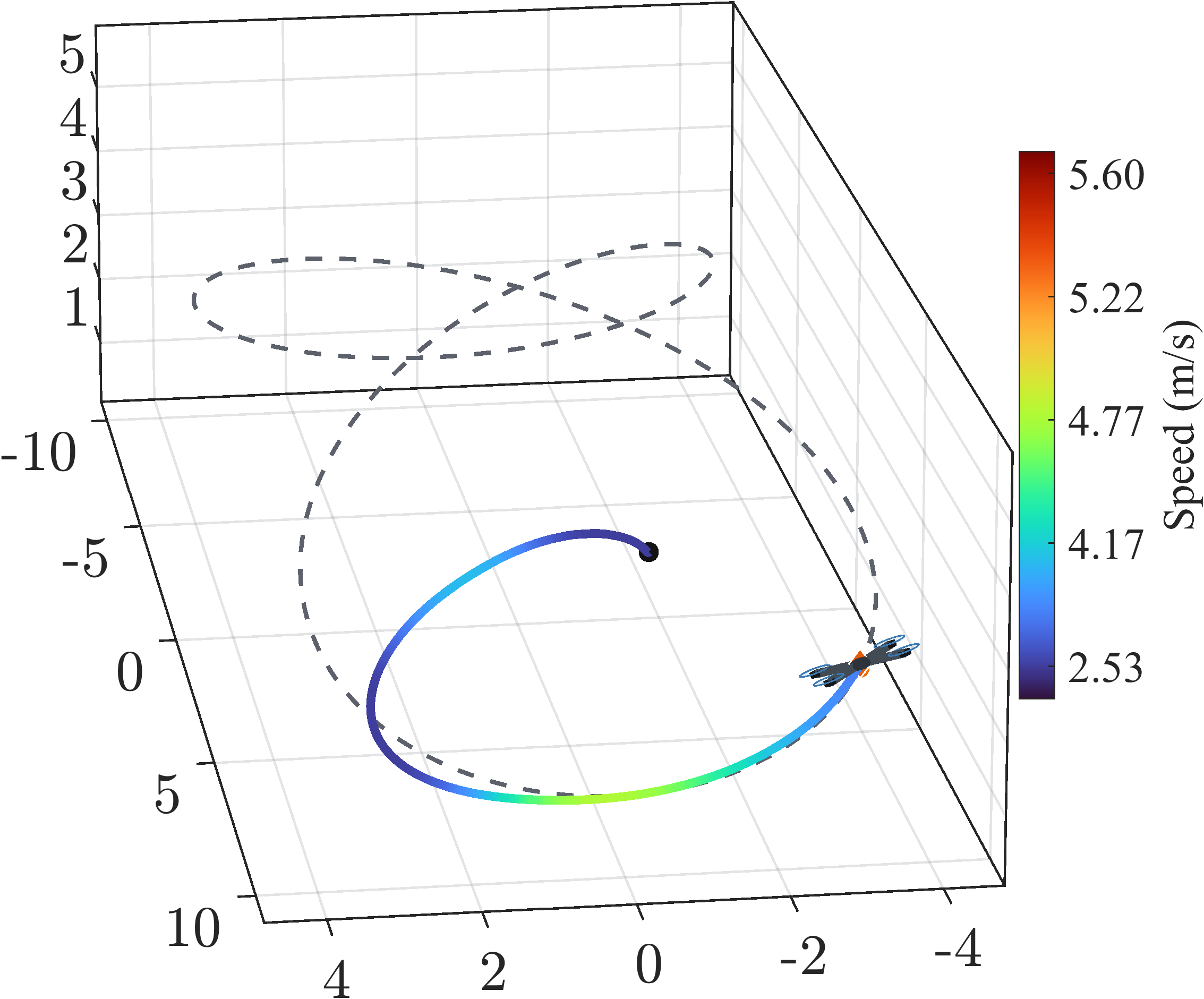}
        \label{fig:008b}
    }
    \subfloat[]{
        \includegraphics[width=0.23\linewidth]{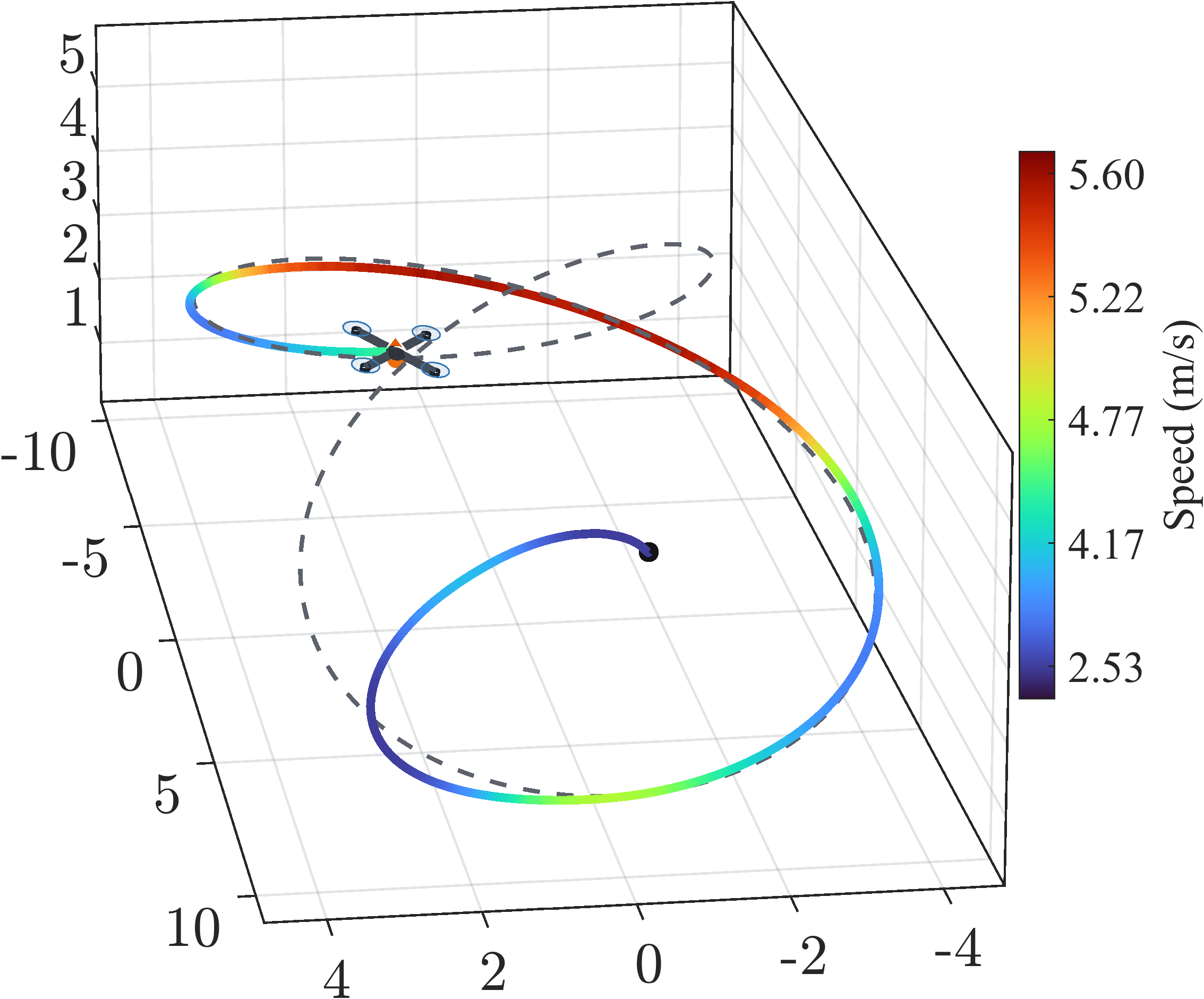}
        \label{fig:008c}
    }
    \subfloat[]{
        \includegraphics[width=0.23\linewidth]{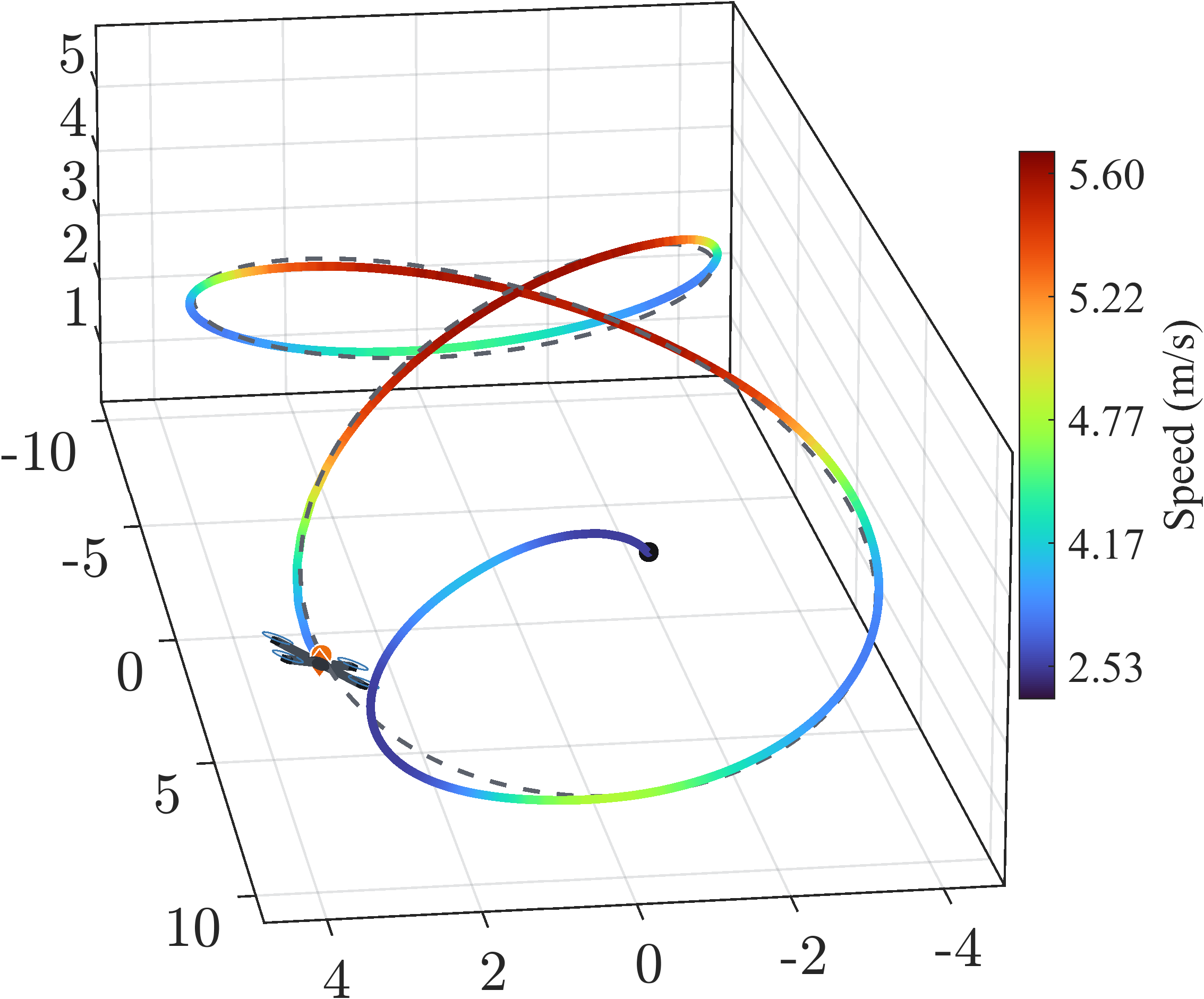}
        \label{fig:008d}
    }
    
    \subfloat[]{
        \includegraphics[width=0.30\linewidth]{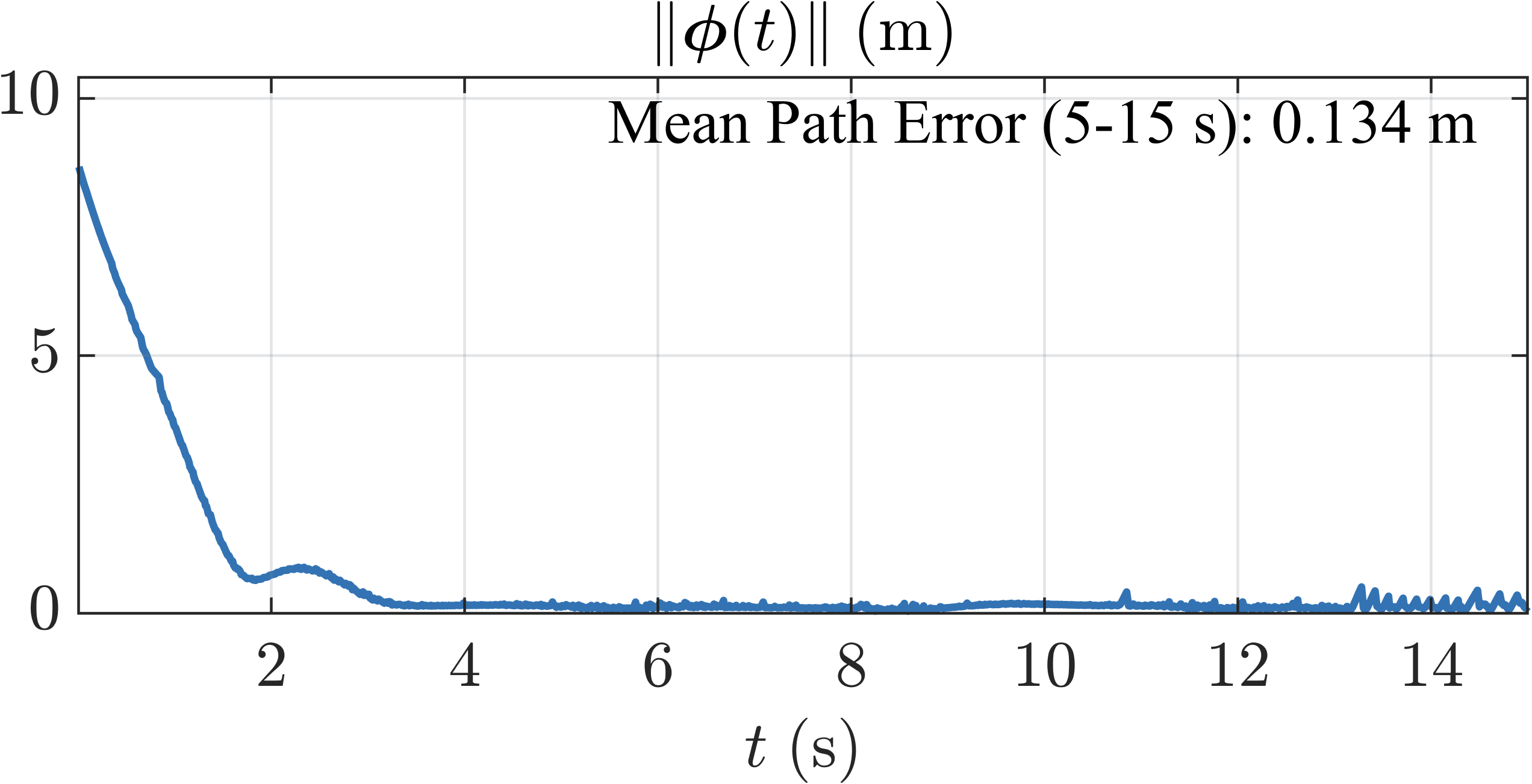}
        \label{fig:008e}
    }
    \subfloat[]{
        \includegraphics[width=0.30\linewidth]{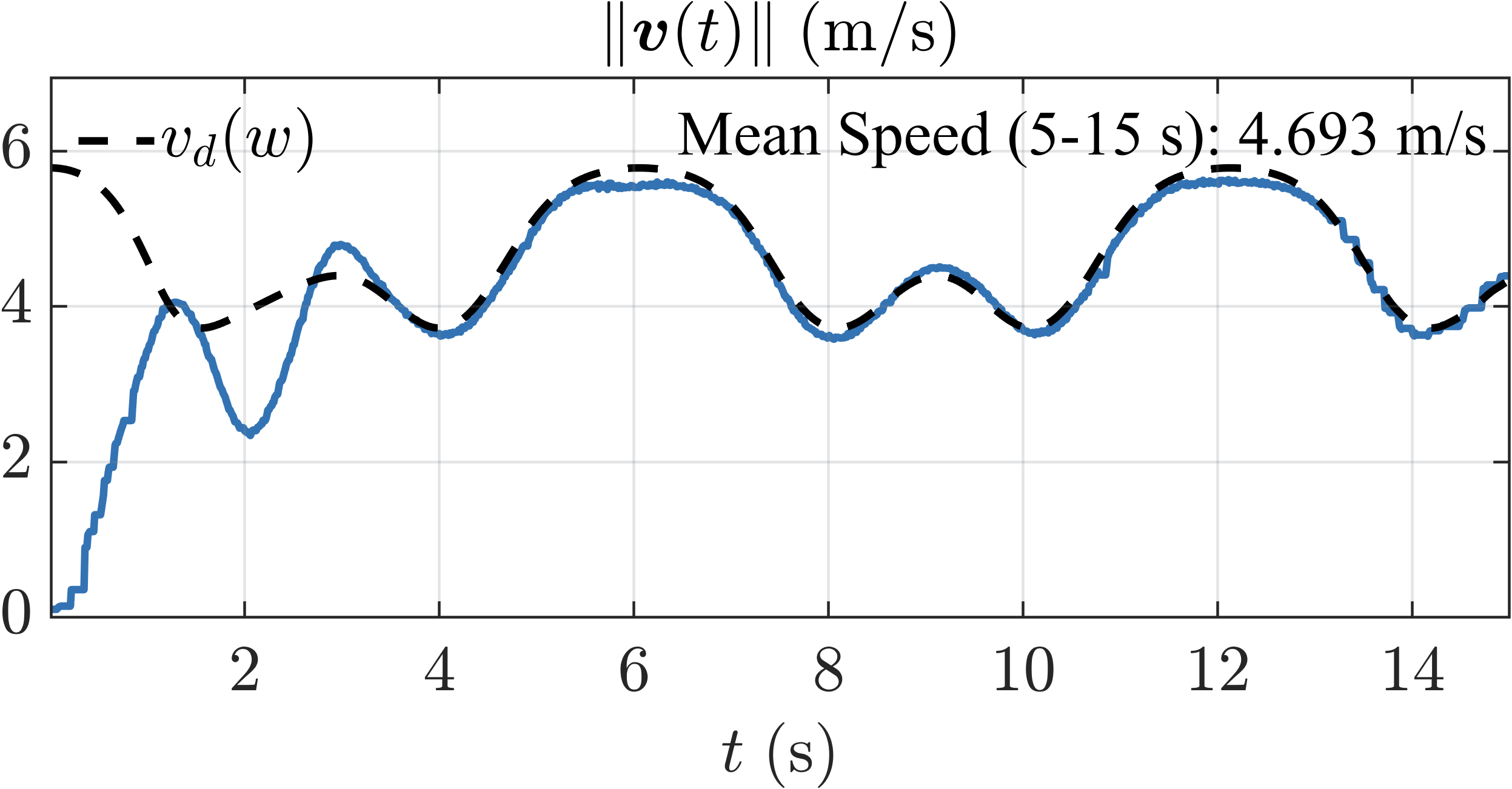}
        \label{fig:008f}
    }
    \subfloat[]{
        \includegraphics[width=0.30\linewidth]{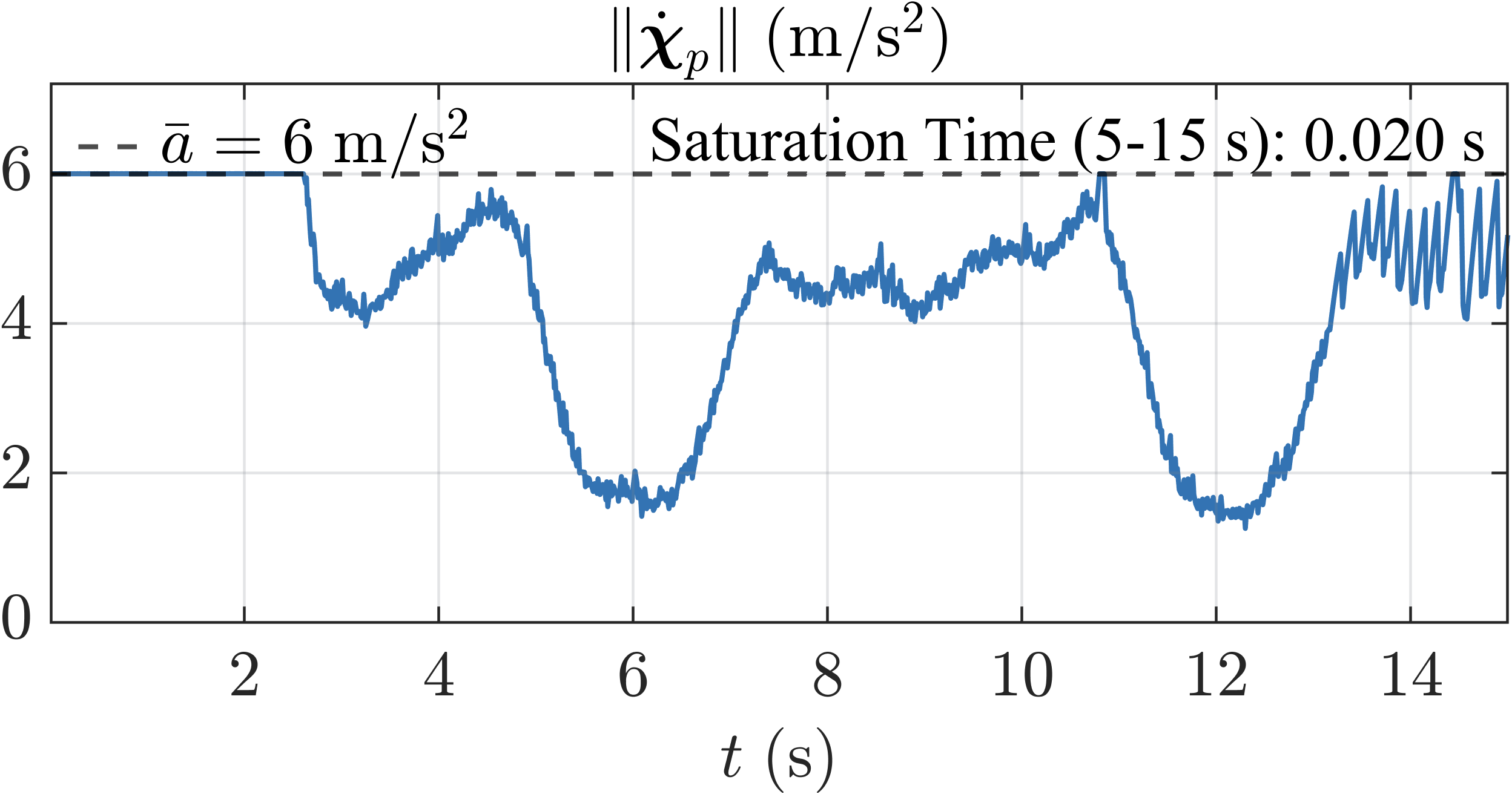}
        \label{fig:008g}
    }
    \caption{Results of the second experiment. (a)–(d) show snapshots at $t=0$, $4$, $9$, and $14\,\mathrm{s}$, respectively, where the gray dashed curve denotes the desired 3D figure “$\infty$” path, and the colored solid curves denote the flight trajectories with the speed magnitude color-coded. (e)–(g) show the path error $\|\boldsymbol{\phi}(t)\|$, the flight speed $\|\boldsymbol{v}(t)\|$ together with PPS $v_d(w)$, and the reference acceleration $\|\boldsymbol{\dot{\chi}}_p\|$, respectively.}
    \label{fig:008}
\end{figure*}

In this experiment, we use a constant PPS for each path and a quadrotor with a takeoff mass of approximately $2.1\,\mathrm{kg}$ to validate the saturated acceleration control algorithm proposed in Theorem~\ref{thm:002}. The quadrotor is equipped with an NxtPX4v2 flight controller running the open-source PX4 firmware (v1.16). The control algorithm runs on a laptop at $100\,\mathrm{Hz}$, and the reference acceleration commands are sent to the quadrotor via Wi-Fi. The pose of the quadrotor is measured by a motion-capture system at $120\,\mathrm{Hz}$, while the experimental data are recorded on the laptop at $100\,\mathrm{Hz}$. We choose the following four curves for path following:
\begin{equation*}
\scalebox{1.0}{$\displaystyle
\begin{aligned}
  \boldsymbol f_1(w)&=
  \begin{bmatrix}
    1.3\cos w-2.6\cos 2w\\
    2\sin w+4\sin 2w\\
    3-2\cos 3w
  \end{bmatrix},
  \boldsymbol f_2(w)=
  \begin{bmatrix}
    2.5\cos 6w\\
    9.5\sin w\\
    3+2\sin 6w
  \end{bmatrix},\\
  \boldsymbol f_3(w)&=
  \begin{bmatrix}
    3.8\sin 2w\\
    9.5\sin w\\
    3+2\cos 2w
  \end{bmatrix},
  \boldsymbol f_4(w)=
  \begin{bmatrix}
    3.9\sin w\\
    4.6\cos w\\
    3+2\sin w
  \end{bmatrix}.
\end{aligned}
$}
\end{equation*}
We set $v_d(w)\equiv v_{d,i}$ with
$v_{d,1}=2.0\,\mathrm{m/s}$,
$v_{d,2}=3.0\,\mathrm{m/s}$,
$v_{d,3}=4.5\,\mathrm{m/s}$, and
$v_{d,4}=5.0\,\mathrm{m/s}$.
The corresponding parameters are chosen as
$v_{s,i}=v_{d,i}+0.25\,\mathrm{m/s}$ and
$\bar v_i=v_{d,i}+0.5\,\mathrm{m/s},i=1,2,3,4$.
The remaining parameters are
$\boldsymbol{K}=2\boldsymbol{I}_3$,
$k_v=3.0$, and
$\bar a=6\,\mathrm{m/s^2}$.
The initial condition for all four experiments is
$\boldsymbol{\xi}(0)=
\begin{bmatrix}
10\,\mathrm{m},0\,\mathrm{m},4\,\mathrm{m},0
\end{bmatrix}^{\top}$.

The flight results are shown in Figs.~\ref{fig:006} and \ref{fig:007}. Although
$\boldsymbol{e}_v^\top\boldsymbol{\dot{\overline{\chi}}}_p\geq0$
is not guaranteed throughout all saturation periods, both
$\|\boldsymbol{\phi}(t)\|$ and $\|\boldsymbol{e}_v(t)\|$
decrease to values close to zero. For a constant PPS, ensuring
$\|\dot{\boldsymbol{\chi}}_p\|\leq\bar a$
along the entire path requires
$v_{d,i}\leq\sqrt{\bar a/\kappa_{\max,i}}$,
where $\kappa_{\max,i}$ is the maximum curvature of the $i$-th path. With
$\bar a=6\,\mathrm{m/s^2}$,
the maximum allowable constant PPSs for
$\boldsymbol f_1$--$\boldsymbol f_4$ are approximately
$2.15$, $3.09$, $4.27$, and $5.01\,\mathrm{m/s}$, respectively. The corresponding experimental values,
$2.0$, $3.0$, $4.5$, and $5.0\,\mathrm{m/s}$,
are close to their respective limits with the PPS for
$\boldsymbol f_3$ slightly exceeding its limit. For $\boldsymbol{f}_3$, reference-acceleration saturation is therefore expected. For the other three paths, reference-acceleration saturation also occurs, as shown in Fig.~\ref{fig:007d}, because unmodeled dynamics and disturbances prevent the conditions
$\|\boldsymbol{\phi}(t)\|=0$ and $\|\boldsymbol{e}_v(t)\|=0$
in Theorem~\ref{thm:003} from being satisfied exactly.

\subsection{ Experiment 2: Path Following with Curvature-Aware PPS}

For the 3D figure “$\infty$” path described by $\boldsymbol f_3$, we have
$L_{\kappa}=\max_{w\in[0,2\pi]} \frac{|\kappa'(w)|}{\|\boldsymbol{f}'_3(w)\|} \approx 0.0676\,\mathrm{m}^{-2}$. We further employ the curvature-aware PPS in \eqref{eq:adaptive_speed} designed in \cref{thm:003}. With the maximum allowable acceleration set to $\bar a=6\,\mathrm{m/s^2}$, the largest cruising speed is
$v_{c,\max}=\left(3\sqrt{3}\bar{a}^2/2L_{\kappa}\right)^{1/4}\approx
6.10\,\mathrm{m/s}$. We choose
$v_c=6\,\mathrm{m/s}$,
$v_s=6.25\,\mathrm{m/s}$, and
$\bar v=6.5\,\mathrm{m/s}$,
while the remaining parameters and the initial condition are the same as those in Section~\ref{sec:exp1}. Here, we focus on the behavior of the quadrotor after it approaches the desired path.

The results are shown in Fig.~\ref{fig:008}. The quadrotor converges to the desired path while adapting its speed to the path curvature. As $\kappa$ increases, $v_d$ and consequently $\|\boldsymbol v\|$ decrease, reducing the required normal acceleration $\kappa\|\boldsymbol v\|^2$. Therefore, reference-acceleration saturation mainly occurs during the initial transient and disappears as the path and velocity-tracking errors decrease. However, motor-response delays, measurement noise, and aerodynamic drag prevent these errors from reaching exactly zero, resulting in a total saturation time of $0.020\,\mathrm{s}$. 

Compared with the constant PPS over $5$--$15\,\mathrm{s}$, the curvature-aware PPS increases the average physical speed by $6.5\%$ while reducing the average path error and acceleration-saturation time by $34.6\%$ and $98.3\%$, respectively, as summarized in Table~\ref{tab:exp_comparison}. These results demonstrate that the curvature-aware PPS improves both path-following accuracy and average motion speed while substantially reducing acceleration saturation.

\begin{table}[h]
    \centering
    \caption{Comparison of two experiments over 5--15\,s for the 3D figure “$\infty$” path.}
    \label{tab:exp_comparison}
    \setlength{\tabcolsep}{5.0pt}
    \renewcommand{\arraystretch}{1.15}
    \begin{tabular}{lccc}
        \toprule
        \textbf{Metric}
        & \textbf{Exp. 1}
        & \textbf{Exp. 2}
        & \textbf{Relative Change} \\
        \midrule
        
        Mean path error (m)
        & 0.205
        & 0.134
        & $\downarrow 34.6\%$ \\
        
        Saturation time (s)
        & 1.155
        & 0.020
        & $\downarrow 98.3\%$ \\
        
        Mean speed (m/s)
        & 4.406
        & 4.693
        & $\uparrow 6.5\%$ \\
        
        \bottomrule
    \end{tabular}
\end{table}

\section{Conclusion}
In this paper, we propose a singularity-free guiding vector field with prescribed physical speed to decouple the robot’s physical speed from the path parameterization. The proposed vector field is globally well-defined and free of off-path singularities. It guarantees global path convergence and physical-speed regulation while yielding path-error dynamics invariant under path reparameterizations. For systems with second-order dynamics subject to acceleration constraints, we further design a saturated controller and a curvature-aware PPS that
maintain the reference acceleration within the prescribed bound under
exact on-path velocity tracking. Comparative simulations and
quadrotor experiments validate the effectiveness of the proposed approach. We have also extended the framework to multi-robot coordinated path following with distributed coordination and collision avoidance, and these results will be reported in future work.

\begin{appendices}
		
	\end{appendices}

	\bibliographystyle{IEEEtran}
	\bibliography{references_corrected.bib}
\end{document}